\documentclass{new_tlp}

\usepackage[utf8]{inputenc}
\usepackage{amsmath}
\usepackage{amssymb}
\usepackage{url}
\usepackage{enumitem}
\usepackage{IEEEtrantools}

\def\wrt{\mbox{w.r.t.}}

\def\ssupporting{\mbox{self-supporting}}
\def\cfree{\mbox{conflict-free}}

\def\A{\mathbf{A}}

\def\PF{\mathbf{P}}

\def\R{\mathbf{R}}

\def\EBAF{\mathbf{E\hspace{-1pt}F}}
\def\EBAFF{\EBAF\hspace{-1pt}=\hspace{-1pt}\tuple{\A\hspace{-1pt},\hspace{-1pt}\R_a\hspace{-1pt},\hspace{-1pt}\R_s\hspace{-1pt},\PF}}

\def\AF{\mathbf{A\hspace{-1pt}F}}
\def\AFF{\AF\hspace{-1pt}=\hspace{-1pt} \tuple{\A\hspace{-1pt},\hspace{-1pt}\R}}
\def\SF{\mathbf{S\hspace{-1pt}F}}
\def\SFF{\SF\hspace{-1pt}=\hspace{-1pt} \tuple{\A\hspace{-1pt},\hspace{-1pt}\R_a}}

\def\aA{E}

\newcommand{\Defeated}[1]{\mathit{Def}(#1)}

\newcommand{\Acceptable}[1]{\mathit{Acc}(#1)}

\newcommand{\UnAcceptable}[1]{\mathit{Un}\hspace{-2pt}\mathit{Acc}(#1)}

\newcommand{\Supported}[1]{\mathit{Sup}(#1)}

\newcommand{\CanSupported}[1]{\mathit{CSup}(#1)}

\newcommand\black{\color{black}}

\def\attacks{\ensuremath{\leadsto}}
\def\att{\ensuremath{\leadsto}}
\def\sup{\ensuremath{\Rightarrow}}
\def\supl{\ensuremath{\Leftarrow}}
\def\Fs{\ensuremath{\mathcal{F}}}
\def\Fsht{\ensuremath{\mathcal{F}_{HT}}}
\def\modelsp{\models^+}
\def\modelsn{\models^-}
\def\modelspm{\models^\pm}

\def\sep{,}
\def\bH{\ensuremath{\mathbf{H}}}
\def\bT{\ensuremath{\mathbf{T}}}

\def\sI{\ensuremath{\mathcal{I}}}
\def\sJ{\ensuremath{\mathcal{J}}}
\def\sK{\ensuremath{\mathcal{K}}}
\def\SI{\ensuremath{E_\sI}}
\def\SJ{\ensuremath{E_\sJ}}
\def\SK{\ensuremath{E_\sK}}

\def\Ninterpretation{\mbox{N-interpretation}}
\def\Ninterpretations{\mbox{N-interpretations}}
\def\HTinterpretation{\mbox{HT-interpretation}}
\def\HTinterpretations{\mbox{HT-interpretations}}
\def\Nmodel{\mbox{N-model}}

\def\HTmodel{\mbox{HT-model}}

\def\fF{\ensuremath{\varphi}}
\def\fG{\ensuremath{\psi}}

\def\larrow{\leftarrow}
\newcommand\citep[1]{\citeN{#1}}
\newcommand\citen[1]{\citeN{#1}}
\newcommand\citeauthor[1]{\citeANP{#1}}
\newcommand\Np{\textbf{N3}}
\newcommand\Nn{\textbf{N4}}
\def\Lan{\mathcal{C}}
\def\LanLP{\Lan_{\hspace{-1pt}L\hspace{-1pt}P}}
\def\LanAF{\Lan_{\hspace{-1pt}A\hspace{-1pt}F}}
\def\LanSF{\Lan_{\hspace{-1pt}S\hspace{-1pt}F}}
\def\LanEF{\Lan_{\hspace{-1pt}E\hspace{-1pt}F}}
\def\LanN{\Lan_{N}}

\newcommand{\LN}[1]{\LanN\hspace{-1pt}(#1)}
\newcommand{\LNp}[1]{\delta#1}
\newcommand{\LNpp}[1]{\tau#1}
\newcommand{\LAF}[1]{\LanAF\hspace{-1pt}(#1)}
\newcommand{\LSF}[1]{\LanSF\hspace{-1pt}(#1)}
\newcommand{\LEF}[1]{\LanEF\hspace{-1pt}(#1)}

\makeatletter
\newcommand*{\attl}{\mathrel{\mathpalette\math@reflect@box\leadsto}}
\newcommand*{\math@reflect@box}[2]{\reflectbox{$#1#2\m@th$}}
\makeatother

\def\piff{\phantom{iff }}
\DeclareMathOperator{\sneg}{\sim\!}
\DeclareMathOperator{\Not}{\mathit{not}\, }

\newcommand{\set}[1]{\ensuremath{\{#1\}}}

\newcommand{\setm}[2]{\ensuremath{\{\ #1\ \big|\ #2\ \}}}

\newcommand{\setBm}[2]{\ensuremath{\Big\{\ #1\ \ \Big|\ \ #2\ \Big\}}}

\newcommand{\tuple}[1]{\ensuremath{\langle #1 \rangle}}

\newcommand{\eqdef}{\mathrel{\vbox{\offinterlineskip\ialign{\hfil##\hfil\cr $\scriptscriptstyle\mathrm{def}$\cr \noalign{\kern1pt}$=$\cr \noalign{\kern-0.1pt}}}}}

\DeclareMathOperator{\veeh}{\mathrel{\vbox{\offinterlineskip\ialign{\hfil##\hfil\cr $\scriptscriptstyle\mathrm{h}$\cr \noalign{\kern-1pt}$\vee$\cr \noalign{\kern-1pt}}}}}

\def\at{At}

\newcommand{\punctfootnote}[1]{\kern-.06em\footnote{#1}}
\newcommand{\punctfootcite}[1]{\kern-.06em\footcite{#1}}

\newcommand{\iiff}{iff}

\usepackage{tikz}
\usetikzlibrary{external}

\usetikzlibrary{fit,backgrounds,shapes,decorations.markings,decorations.pathmorphing,arrows,shadows,patterns,shapes.callouts,calc,intersections}
\tikzstyle{arg}        = [circle, thick, minimum size=0.2cm, draw=black, font=\scriptsize, fill=white, preaction={
  fill=gray,
  transform canvas={xshift=0.8mm,yshift=-0.8mm}
}]
\tikzstyle{narg}        = [circle, dashed, thick, minimum size=0.2cm, draw=black, font=\scriptsize, fill=white, preaction={
  fill=gray,
  transform canvas={xshift=0.8mm,yshift=-0.8mm}
}]
\tikzstyle{vide}        = [thick, minimum size=0.2cm, draw=black, font=\scriptsize, fill=white, preaction={
  fill=gray,
  transform canvas={xshift=0.8mm,yshift=-0.8mm}
}]
\tikzstyle{nvide}        = [thick, dashed, minimum size=0.2cm, draw=black, font=\scriptsize, fill=white, preaction={
  fill=gray,
  transform canvas={xshift=0.8mm,yshift=-0.8mm}
}]
\tikzstyle{videphantom}        = [thick, minimum size=0.2cm, draw=white, font=\scriptsize, fill=white, preaction={
  fill=white,
  transform canvas={xshift=0.8mm,yshift=-0.8mm,
  general shadow={fill=white,shadow scale=2}}
}]
\tikzstyle{ttvide}        = [thick, minimum size=0.01cm, draw=white, font=\normalsize, fill=white
]
\tikzstyle{tikzpict}   = [>=latex,text depth=0.25ex]

\newcounter{programcount}

\newcommand{\newprogram}{\refstepcounter{programcount}\ensuremath{P_{\arabic{programcount}}}}

\newcommand{\program}[2]{\ensuremath{P_{#1{#2}}}}
\newcommand{\newtheory}{\refstepcounter{programcount}\ensuremath{\Gamma_{\arabic{programcount}}}}
\newcommand{\theory}[2]{\ensuremath{\Gamma_{#1{#2}}}}
\newcommand{\newaf}{\refstepcounter{programcount}\ensuremath{\AF_{\arabic{programcount}}}}
\newcommand{\af}[2]{\ensuremath{\AF_{#1{#2}}}}
\newcommand{\newef}{\refstepcounter{programcount}\ensuremath{\EBAF_{\arabic{programcount}}}}
\newcommand{\ef}[2]{\ensuremath{\EBAF_{#1{#2}}}}
\newcommand{\SigmaP}[1]{\ensuremath{\mathrm{\Sigma}^{\mathrm{P}}_{#1}}}

\newenvironment{examplecont}[1]{\begin{example}[Ex.~\ref{#1} continued]}{\end{example}}
\newenvironment{proofs}{\trivlist\item\ignorespaces\noindent\itshape Proof sketch.\normalfont}{\qed\endtrivlist}

\newenvironment{proofof}[1]{\normalfont\rmfamily \trivlist
  \pagebreak[3]\item[\hskip \labelsep {\normalfont\bf Proof of #1}.]}
  {\hspace*{1em}\qed\endtrivlist}

\providecommand{\qed}{\hfill\proofbox}
\newtheorem{theorem}{Theorem}
\newtheorem{fact}{Fact}
\newtheorem{lemma}[fact]{Lemma}
\newtheorem{corollary}{Corollary}
\newtheorem{proposition}{Proposition}
\newtheorem{definition}{Definition}
\newtheorem{example}{Example}

\title[AFs and ASP are two faces of Nelson’s logic]{Abstract argumentation and answer set programming: two faces of Nelson’s logic\thanks{This work is an extended version of a paper presented at the  Sixteenth International Conference on Principles of Knowledge Representation and Reasoning and entitled \emph{Constructive Logic Covers Argumentation and Logic Programming}~\protect\cite{fanfar18a}.}}

\author[Jorge Fandinno and Luis Fari{\~n}as del Cerro]{
  \MakeUppercase{Jorge Fandinno}\\
  University of Nebraska Omaha, USA,\\
  \email{jfandinno@unomaha.edu}
  \and
  \MakeUppercase{Luis Fari{\~n}as del Cerro}\\
  IRIT, Universit{\'e} de Toulouse, CNRS, Toulouse, France,\\
  \email{luis@irit.fr}
}

\begin{document}

\maketitle

\begin{abstract}
In this work, we show that both logic programming and abstract argumentation frameworks can be interpreted in terms of Nelson's constructive logic N4. We do so by formalising, in this logic, two principles that we call non-contradictory inference and strengthened closed world assumption: the first states that no belief can be held based on contradictory evidence while the latter forces both unknown and contradictory evidence to be regarded as false. Using these principles, both logic programming and abstract argumentation frameworks are translated into constructive logic in a modular way and using the object language. Logic programming implication and abstract argumentation supports become, in the translation, a new implication connective following the non-contradictory inference principle. Attacks are then represented by combining this new implication with strong negation.
\emph{Under consideration in Theory and Practice of Logic Programming (TPLP)}
\end{abstract}
 
\section{Introduction}\label{sec:introduction}
 
Logic programming~(LP) and Abstract Argumentation Frameworks~(AFs) are two well-established formalisms for
Knowledge Representation and Reasoning~(KR) 
whose close relation is well-known since the introduction of the latter: besides introducing AFs, \citen{Dung95} studied how logic programs under the \emph{stable models}~\cite{GL88} and the \emph{well-founded semantics}~\cite{van1991well} can be translated into
abstract argumentation frameworks.
Since then, this initial connection has been further studied and extended,
providing relations between other semantics and ways to translate argumentation frameworks into logic programs~\cite{NievesCO08,Caminada2009,Wu2010ALJ,Toni2011,DvorakGWW11,CaminadaSAD15}.

On the other hand, Nelson's \emph{constructive logic}~\cite{nelson1949} is a conservative extension of \emph{intuitionistic logic}, which introduces the notion of \emph{strong negation} as a means to deal with constructive falsity, in an analogous way as intuitionism deals with constructive truth.
\citeauthor{Pearce96}~\citeyear{Pearce96,Pearce06}
showed that a particular selection of models of constructive logic, called \emph{equilibrium logic}, precisely characterize the stable models of a logic program.
This characterization was later extended to the \emph{three\nobreakdash-valued stable model}~\cite{przymusinski91a} and the well-founded semantics by~\citeN{caodpeva07a}.
Versions of constructive logic without the ``explosive'' axiom \mbox{$\varphi \to (\sneg\varphi \to \psi)$} have been extensively studied in the literature~\cite{nelson1959negation,escobar1972,thomason1969semantical,nelson1984,Odintsov2005,odintsov2015inference,kamide2015proof}
and can be considered a kind of \emph{paraconsistent} logics, in the sense, that some formulas may be constructively true and false at the same time.
The notion of equilibrium has been extended to one of these logics by \citen{OdintsovP05}, who also showed that this precise characterize the \emph{paraconsistent stable semantics}~\cite{SakamaI95}.

In this paper, we formalize in Nelson's constructive logic 
a reasoning principle, to be called \emph{non-contradictory inference} (denoted~\ref{p:contradictory}), which states that 
\begin{enumerate}[ itemindent=-4pt, leftmargin=29pt, rightmargin=15pt, label=\textbf{NC} ]
\item ``\emph{no belief can be held based on contradictory evidence.}''
  \label{p:contradictory}
\end{enumerate}
Interestingly, though different from the logic studied by~\citeauthor{OdintsovP05},
the logic presented here is also
a conservative extension of equilibrium logic (and, thus, also of LP under the stable models semantics) that allows us to deal with inconsistent information in LP.
The interesting feature of this new logic is that, besides LP, it also captures several classes of AFs, under the stable semantics.
It is worth to mention that the representation of AFs in this new logic is modular and it is done
 using an \emph{object language level}.
Recall that by object language level, we mean that AFs and its logical translation \emph{share the same language} (each argument in the AF becomes an atom in its corresponding logical theory)
and the relation between arguments in the AF (attacks or supports) are expressed by means of logical connectives.
This contrast  with \emph{meta level approaches}, which talk about the AFs from ``above,'' using another language and relegating logic to talk about this new language.
It is important to note that, as highlighted by~\citen{Gabbay2015TheAA}, 
the object language oriented approaches
have the remarkable property of providing alternative intuitive meaning to the translated concepts through their interpretation\ in logic.
In this sense,
from the viewpoint of constructive logic, AFs can be understood as a 
\emph{strengthened closed world assumption}~\cite{reiter1980logic} that we denote as~\mbox{\ref{p:cwa}}:
\begin{enumerate}[ itemindent=-4pt, leftmargin=30pt, rightmargin=15pt, label=\textbf{C\hspace{-1.5pt}W}, start=2 ]
\item ``\emph{everything for which we do not have evidence of being true or for which we have contradictory evidence, should be regarded as false}''
\label{p:cwa}
\end{enumerate}

The relation between AFs and logic has been extensively studied in the literature
and, as mentioned above, can be divided in two categories:
those that follow an object language approach~\cite{Caminada2009,Gabbay2015TheAA,gabbay2016attack}
and those that follow a meta level approach~\cite{BesnardD04,Caminada2009,Grossi2011,DvorakSW12,Ariel2013,DoutreHP14,BesnardDH14,DDvorakGLW14}.
In particular, the approach we take here shares with the work by~\citeN{Gabbay2015TheAA} the use of strong negation to capture attacks, but differs in the underlying logic: constructive logic in our case and classical logic in the case of~\citeauthor{Gabbay2015TheAA}'s work.
On the intuitive level, under the constructive logic point of view, \emph{attacks} can be understood as
\begin{enumerate}[ itemindent=-4pt, leftmargin=26pt, rightmargin=18pt, label=\textbf{AT}, start=2 ]
\item ``\emph{means to construct a proof of the falsity of the attacked argument based on the acceptability of the attacker}''
  \label{p:attack}
\end{enumerate}
On the practical level, the use of constructive logic allows for a more \emph{compact} and \emph{modular translation}: each attack becomes a (rule-like) formula with the attacker -- or a conjunction of attackers in the case of set attacking arguments~\cite{Nielsen2007} -- as the antecedent and the attacked argument as the consequent.
Moreover, when attacks are combined with LP implication, we show that the latter captures the notion of \emph{support} in Evidential-Based Argumentation Frameworks (EBAFs;~\citeNP{OrenN08}):
for accepting an argument, these frameworks require, not only its \emph{acceptability} as in Dung's sense, but also that it is supported by some chain of supports rooted in a kind of special arguments called \emph{prima-facie}.

\section{Background}

In this section we recall the needed background regarding Nelson's constructive logic, logic programming and argumentation frameworks.

\subsection{Nelson's Constructive Logic}

The concept of constructive falsity was introduced into logic by~\citen{nelson1949} and it is often denoted
as $\Np$.
It was first axiomatized by \citen{vorob1952constructive}, and later studied by \citen{markov1953constructive}, who related intuitionistic and strong negation, and by~\citen{rasiowa1969n}, who provided an algebraic characterization.
Versions of constructive logic without the ``explosive'' axiom
\mbox{$\varphi \to (\sneg\varphi \to \psi)$}
are usually denoted as~$\Nn$
and they are based on a four valued assignment for each world corresponding to the values \emph{unknown}, \emph{(constructively) true}, \emph{(constructively) false} and \emph{inconsistent} (or \emph{overdetermined}).
The logic $\Np$ can be obtained by adding back the ``explosive'' axiom.
We describe next a Kripke semantics for a version of $\Nn$~\cite{thomason1969semantical,gurevich1977intuitionistic}
with the falsity constant~$\bot$,
which is denoted as $\Nn^\bot$ by~\citeN{odintsov2015inference}.
We follow here an approach with two forcing relations in the style of the work by~\citeN{akama1987}.
An alternative characterization using $2$-valued assignments plus an involution has been described by~\citeN{routley1974semantical}.

Syntactically, we assume a logical language with a \emph{strong negation} connective~``$\sneg$''.
That is, given some (possibly infinite) set of atoms $\at$,
a \emph{formula} $\fF$ is defined using the grammar:
\[
\fF \quad ::= \quad \bot \ \mid \ a \ \mid \ \sneg \fF \ \mid \ \fF \wedge \fF \ \mid \ \fF \vee \fF \ \mid \ \fF \to \fF
\]
with \mbox{$a \in \at$}.
We use Greek letters $\fF$ and $\fG$ and their variants to stand for propositional formulas.
\emph{Intuitionistic negation} is defined as ${\neg \fF \eqdef (\fF \to \bot)}$.
We also define the derived operators
${\fF \leftrightarrow \fG \eqdef (\fF \to \fG) \wedge (\fG \to \fF)}$
and
${\top \eqdef \sneg \bot}$.

A Kripke frame $\Fs = \tuple{W,\leq}$ is a pair
where $W$ is a
non-empty
set of worlds and $\leq$ is a partial order on $W$.
A valuation ${V : W \longrightarrow  2^{\at}}$ is a function mapping each world to a subset of atoms.
A Nelson's interpretation (\Ninterpretation) is a
3-tuple ${\sI = \tuple{\Fs,V^+,V^-}}$
where ${\Fs = \tuple{W,\leq}}$ is a Kripke frame and
where both $V^+$ and $V^-$ are valuations
satisfying, for every pair of worlds
$w,w' \in W$ with $w \leq w'$ and every atom $a \in \at$, the following preservation properties:
\begin{enumerate}[ label=\roman*) , leftmargin=20pt ]
\item $V^+(w) \subseteq V^+(w')$, and
\item $V^-(w) \subseteq V^-(w')$.
\end{enumerate}
Intuitively, $V^+$ represents our knowledge about constructive truth while
$V^-$ represents our knowledge about constructive falsity.
We say that $\sI$ is \emph{consistent} if, in addition, it satisfies:
\begin{enumerate}[ label=\roman*), start=3 , leftmargin=20pt]
\item $V^+(w) \cap V^-(w) = \varnothing$ for every world $w \in W$.
\end{enumerate}
Two forcing relations $\modelsp$ and $\modelsn$ are defined
with respect to any \Ninterpretation\ ${\sI = \tuple{\Fs,V^+,V^-}}$,
world $w \in W$ and atom $a \in \at$ as follows:
\begin{IEEEeqnarray*}{l ?C? l }
\sI\sep w \modelsp a &\text{ iff }& a \in V^+(w)
\\
\sI\sep w \modelsn a &\text{ iff }& a \in V^-(w)
\end{IEEEeqnarray*}
These two relations are extended to compounded formulas as follows:
\begin{IEEEeqnarray*}{l ,C, l }
\sI\sep w  \not\modelsp \bot
\\
\sI\sep w \modelsp \varphi_1 \wedge \varphi_2
    &\text{ iff }& \sI\sep w \modelsp \varphi_1 \text{ and } \sI\sep w \modelsp  \varphi_2
\\
\sI\sep w \modelsp \varphi_1 \vee \varphi_2
    &\text{ iff }& \sI\sep w \modelsp \varphi_1 \text{ or } \sI\sep w \modelsp  \varphi_2
\\
\sI\sep w \modelsp \varphi_1 \!\to\! \varphi_2
    &\text{ iff }& \forall w'\!\geq\! w \ \sI\sep w' \!\not\modelsp\! \varphi_1 \hspace{-0.5pt}\text{ or } \sI\sep w' \!\modelsp\!  \varphi_2
\\
\sI\sep w \modelsp \sneg\varphi
     &\text{ iff }& \sI\sep w \modelsn\varphi
\\
\sI\sep w  \modelsn \bot
\\
\sI\sep w \modelsn \varphi_1 \wedge \varphi_2
    &\text{ iff }& \sI\sep w \modelsn \varphi_1 \text{ or } \sI\sep w \modelsn  \varphi_2
\\
\sI\sep w \modelsn \varphi_1 \vee \varphi_2
    &\text{ iff }& \sI\sep w \modelsn \varphi_1 \text{ and } I\sep w \modelsn  \varphi_2
\\
\sI\sep w \modelsn \varphi_1 \!\to\! \varphi_2
    &\text{ iff }&  \sI\sep w \modelsp \varphi_1  \text{ and } \sI\sep w \modelsn  \varphi_2
\\
\sI\sep w \modelsn \sneg\varphi
     &\text{ iff }& \sI\sep w \modelsp\varphi
\end{IEEEeqnarray*}
An \Ninterpretation\ is said to be an \emph{\Nmodel} of a formula~$\varphi$,
in symbols $\sI \modelsp \varphi$, iff $\sI \sep w \modelsp \varphi$ for every $w \in W$.
It is said to be \emph{\Nmodel} of a theory~$\Gamma$,
in symbols also ${\sI \modelsp \Gamma}$, iff it is an \Nmodel\ of all its formulas
${\sI \modelsp \varphi}$.
A formula $\varphi$ is said to be a \emph{consequence} of a theory~$\Gamma$
iff every model of $\Gamma$ is also a model of $\varphi$,
that is $\sI \modelsp \varphi$ for every $\sI \modelsp \Gamma$.
This formalization characterizes $\Nn$ while a restriction to consistent \Ninterpretation s would characterize $\Np$.
As mentioned above, $\Nn$ is ``somehow'' paraconsistent in the sense that a formula $\varphi$ and its strongly negated counterpart~$\sneg\varphi$ may simultaneously be consequences of some theory:
for instance,
we have that $\set{a, \sneg a} \modelsp a$ and $\set{a, \sneg a} \modelsp \sneg a$.
Intuitively, these two forcing relations determine the four values above mentioned:
a formula $\varphi$ satisfying $\sI \not\modelsp \varphi$ and $\sI \not\modelsn \varphi$
is understood as \emph{unknown}.
If it satisfies
$\sI \modelsp \varphi$ and $\sI \not\modelsn \varphi$,
is understood as \emph{true}.
\emph{False} if $\sI \not\modelsp \varphi$ and $\sI \modelsn \varphi$,
and \emph{inconsistent} if $\sI \modelsp \varphi$ and $\sI \modelsn \varphi$.

\subsection{Logic Programming, Equilibrium Logic and\\ Here-and-There Nelson's Models}

In order to accommodate logic programming conventions,
we will indistinctly write ${\varphi \larrow \psi}$ instead of ${\psi \to \varphi}$ when describing logic programs.
An \emph{explicit literal} is either an atom~\mbox{$a \in \at$} or an atom preceded by strong negation~$\sneg a$.
A \emph{literal} is either an explicit literal~$l$ or an explicit literal preceded by intuitionistic negation~$\neg l$.
A literal that contains intuitionistic negation is called \emph{negative}.
Otherwise, it is called \emph{positive}.
A \emph{rule} is a formula of the form \mbox{$H \larrow B$}
where $H$ is a disjunction of atoms and $B$ is a conjunction of literals.
A logic program~$\Pi$ is a set of rules.

Given some set of explicit literals $\bT$ and some formula $\varphi$,
we write ${\bT \modelsp \varphi}$ when \mbox{$\tuple{\Fs,V^+,V^-} \modelsp \varphi$} holds
for the Kripke frame~$\Fs$ with a unique world~$w$ and valuations:
\mbox{$V^+(w) = \bT \cap \at$}
and
\mbox{$V^-(w) = \setm{ a }{ \sneg a \in \bT}$}.
A set of explicit literals~$\bT$ is said to be \emph{closed} under $\Pi$ if $\bT \modelsp H \larrow B$ for every rule $H \larrow B$
in~$\Pi$.

Next, we recall the notions of reduct and answer set~\cite{gelfondL91}:

\begin{definition}[Reduct and Answer Set]\label{def:answer.set}
The \emph{reduct} of program $\Pi$ \wrt~some set of explicit literals~$\bT$ is defined as follows
\begin{enumerate}[ label=\it\roman*) , leftmargin=20pt ]
\item Remove all rules with $\neg l$ in the body s.t. $l \in \bT$,
\item Remove all negative literals for the remaining rules.
\end{enumerate}
Set~$\bT$ is a \emph{stable model} of $\Pi$ if $\bT$ is a $\subseteq$-minimal closed set under~$\Pi$.
\end{definition}

For characterizing logic programs in constructive logic,
\black
we are only interested in a particular kind of {\Ninterpretation s} over \mbox{\emph{Here-and-There}}~(HT) frames.
These frames are of the form \mbox{$\Fsht = \tuple{\set{h,t}, \leq }$}
where $\leq$ is a partial order satisfying \mbox{$h \leq t$}.
We refer to \Ninterpretation s\ with an HT-frame as \emph{\HTinterpretation s}.
A \emph{\HTmodel} is an \Nmodel\ which is also a \HTinterpretation.
We use the generic terms \emph{interpretation} (resp. \emph{model}) for both HT and \Ninterpretations\ (resp. models) when it is clear by the context.
At first sight, it may look that restricting ourselves to HT frames is an oversimplification.
However, once the closed world assumption is added to intuitionistic logic, this logic can be replaced without loss of generality by any proper intermediate logic~\cite{OsorioPA05,CabalarFC0V17}.

Given any \HTinterpretation,
\mbox{$\sI = \tuple{\Fsht,V^+,V^-}$}
we define
four sets of atoms as follows:
\begin{gather*}
\begin{IEEEeqnarraybox}{ l ,C, l }
H_\sI^+ &\eqdef& V^+(h)\\
H_\sI^- &\eqdef& V^-(h)
\end{IEEEeqnarraybox}
\hspace{2cm}
\begin{IEEEeqnarraybox}{ l ,C, l }
T_\sI^+ &\eqdef& V^+(t)\\
T_\sI^- &\eqdef& V^-(t)
\end{IEEEeqnarraybox}
\end{gather*}
These sets of atoms correspond to the atoms verified at each corresponding world and valuation.
Every \HTinterpretation~$\sI$ is fully determined by these four sets.
We will omit the subscript and write, for instance, $H^+$ instead of $H^+_\sI$ when $\sI$ is clear from the context.
Furthermore,
any \HTinterpretations\ can be succinctly rewritten as a pair
\mbox{$\sI = \tuple{\bH,\bT}$}
where 
\mbox{$\bH = H^+ \cup \sneg H^-$}
and
\mbox{$\bT = T^+ \cup \sneg T^-$}
are sets of literals.\punctfootnote{We denote by $\sneg S \eqdef \setm{ \sneg \varphi }{ \varphi \in S}$ the of set strongly negated formulas of a given set~$S$.
Similarly, we also define
$\neg S \eqdef \setm{ \neg \varphi }{ \varphi \in S}$.}
Note that, by the preservation properties of \Ninterpretations, we have that ${\bH \subseteq \bT}$.
We say that an \HTinterpretation~\mbox{$\sI = \tuple{\bH,\bT}$}
is \emph{total} iff $\bH = \bT$.
Given \HTinterpretations\ $\sI = \tuple{\bH,\bT}$ and $\sI' = \tuple{\bH',\bT'}$,
we write $\sI \leq \sI'$ iff $\bH \subseteq \bH'$ and $\bT = \bT'$.
As usual, we write $\sI < \sI'$ iff $\sI \leq \sI'$ and $\sI \neq \sI'$.

Next, we introduce the definition of equilibrium model~\cite{Pearce96}.

\begin{definition}[Equilibrium model]\label{def:equilibrium}
A \HTmodel~$\sI$ of a theory $\Gamma$ is said to be an \emph{equilibrium model} iff it is total and there is no other \HTmodel~$\sI'$ of $\Gamma$ s.t. $\sI' < \sI$.
\end{definition}

Interestingly, consistent equilibrium models precisely capture the answer set of a logic program.
The following is a rephrase of Proposition~2 by~\citeN{Pearce96} using our notation.

\begin{proposition}\label{prop:per96}
Let~$\Pi$ be a logic program.
A consistent set~$\bT$ of explicit literals is a stable model of~$\Pi$ if and only if~$\bT$ is the set of explicit literals true in some consistent equilibrium model of~$\Pi$.
\end{proposition}

More in general, it has been shown by~\citeN{OdintsovP05} that the 
(possible non-consistent) equilibrium models of a logic program capture its paraconsistent answer sets~\cite{SakamaI95}.

The following propositions characterizes some interesting properties of HT and strong negation that will be useful through the paper\footnote{For the sake of clarity, proofs of formal results are moving to an appendix.}
:

\begin{proposition}[Persistence]\label{prop:preservation}
Any \HTinterpretation~$\sI$, formula~$\varphi$ and world \mbox{$w \in \set{h,t}$} satisfy:
\begin{enumerate}
\item $I\sep w \modelsp \varphi$ implies $I\sep t \modelsp \varphi$, and
\item $I\sep w \modelsn \varphi$ implies $I\sep t \modelsn \varphi$.
\end{enumerate}
\end{proposition}

\begin{proposition}[HT-negation]{\label{prop:negation}}
Any \HTinterpretation~$\sI$, formula~$\varphi$ and world \mbox{$w \in \set{h,t}$} satisfy:
\begin{enumerate}[ label=\roman*) , leftmargin=20pt ]
\item $\sI\sep w \modelsp \neg\varphi$ iff $\sI\sep t \not\modelsp \varphi$, and
\label{item:1:prop:negation}

\item $\sI\sep w \modelsp \neg\neg\varphi$ iff $\sI\sep t \modelsp \varphi$, and
\label{item:5:prop:negation}

\item $\sI\sep w \modelsp \neg\neg\neg\varphi$
iff $\sI\sep w \modelsp \neg\varphi$, and
\label{item:7:prop:negation}

\item $\sI\sep w \modelsn \neg\varphi$ 
iff $\sI\sep w \modelsn \sneg\varphi$.
\label{item:2:prop:negation}
\end{enumerate}
\end{proposition}

\subsection{Abstract Argumentation Frameworks}

Since their introduction, the syntax of AFs have been extended in different ways.
One of these extensions, usually called SETAFs, consists in generalizing the notion of binary attacks to collective attacks such that a set of arguments~$B$ attacks some argument $a$~\cite{Nielsen2007}.
Another such extension, usually called Bipolar~AFs~(BAFs), consists in frameworks with a second positive relation called \emph{support}~\cite{karacapilidis2001computer,verheij2003deflog,AmgoudCL04}.
In particular,~\citen{Verheij03}
introduced the idea that, in AFs, arguments are considered as \emph{prima-facie} justified statements, which can be considered true until proved otherwise, that is, until they are defeated.
This allows introducing a second class of \emph{ordinary arguments}, which cannot be considered true unless get supported by the prima-facie ones.
Later, 
\citen{Oren2014}
developed this idea by introducing Evidence-Based AFs (EBAFs),
an extension of SETAFs (and, this, of AFs) which incorporates the notions of support and prima-facie arguments.
Next we introduce an equivalent definition by~\citeN{Fandinno2018foiks},
which is closer to the logic formulation we pursue here.

\begin{definition}[Evidence-Based Argumentation framework]\label{def:EBA}
An Evidence-Based Argumentation framework $\EBAFF$ is a \mbox{$4$-tuple}
where
$\A$ represents a (possibly infinite) set of arguments,
\mbox{$\R_a \subseteq 2^{\A} \times {\A}$}
is an attack relation,
\mbox{$\R_s \subseteq 2^{\A} \times \A$} is a support relation
and $\PF \subseteq \A$ is a set of distinguished \emph{prima-facie} arguments.
We say that an $\EBAF$ is \emph{finitary} iff $B$ is finite for every attack or support
\mbox{$(B,a) \in \R_a \cup \R_s$}. 
\end{definition}

The notion of acceptability is extended by requiring not only defense against all attacking arguments, but also support from some prima-facie arguments.
Furthermore, the defense can be provided not only by defeating all attacking sets of arguments, but also by denying the necessary support for some of the non-prima-facie arguments of these attacks.

\begin{definition}[Defeat/Acceptability]\label{def-inh-ASAF}
Given some argument \mbox{$a \in \A$} and set of arguments $\aA \subseteq \A$, we say
\begin{enumerate}[ topsep=3pt, itemsep=0pt, parsep=1pt, start=1, leftmargin=15pt ]
\item $a$ is \emph{defeated} \wrt\ $\aA$ iff
there is some $B \subseteq \aA$ s.t. $(B,a) \in \R_a$,
\end{enumerate}
$\Defeated{\aA}$ 
will denote the set of arguments that are defeated \wrt~$\aA$.
\begin{enumerate}[ topsep=3pt, itemsep=0pt, parsep=1pt, start=2, leftmargin=15pt ]
\item $a$ is \emph{supported} \wrt\ $\aA$ iff either \mbox{$a \in \PF$} or there is some 
\mbox{$B \subseteq \aA \setminus\set{a}$} whose elements are supported \wrt\ $\aA \setminus\set{a}$ and such that \mbox{$(B,a) \in \R_s$},

\item $a$ is \emph{supportable} \wrt~$\aA$ iff it is supported \wrt~$\A \setminus \Defeated{\aA}$,

\item $a$ is \emph{unacceptable} \wrt\ $\aA$ iff it is either defeated or not supportable,

\item $a$ is \emph{acceptable} \wrt\ $\aA$ iff it is supported and,
  for every $(B,a) \in \R_a$, there is $b \in B$ such that $b$ is unacceptable \wrt\ $\aA$
\end{enumerate}
$\Supported{\aA}$ (resp. $\UnAcceptable{\aA}$ and $\Acceptable{\aA}$)
will denote the set of arguments that are supported (resp. unacceptable and acceptable) \wrt~$\aA$.
\end{definition}

Then, semantics are defined as follows:

\begin{definition}\label{def:semantics}
A set of arguments $\aA \subseteq \A$
is said to be:
\begin{enumerate}[ topsep=2pt, itemsep=0pt, parsep=1pt, start=1 , leftmargin=15pt ]
\item \emph{\ssupporting} \iiff\ $\aA \subseteq \Supported{\aA}$,

\item \emph{\cfree}
\iiff\
$\aA \!\cap\! \Defeated{\aA} \!=\! \varnothing$,

\item \emph{admissible} \iiff\ it is \cfree\ and $\aA \subseteq \Acceptable{\aA}$,

\item \emph{complete} \iiff\ it is \cfree\ and
$\aA = \Acceptable{\aA}$,

\item \emph{preferred} \iiff\ it is a $\subseteq$-maximal admissible set,

\item \emph{stable}
iff
\mbox{$\aA = \A \setminus \UnAcceptable{\aA}$}.
\end{enumerate}
\end{definition}

\noindent
SETAFs can be seen as special cases where the set of supports is empty and all arguments are prima-facie.
In this sense, we write \mbox{$\SFF$}
instead
\mbox{$\EBAF = \tuple{\A,\R_a,\varnothing,\A}$}.
Furthermore, in their turn, AFs can be seen as a special case of SETAFs where all attacks have singleton sources.
In such case, we just write $\mbox{$\AFF$}$
instead $\SFF$,
where \mbox{$\R = \setm{ (b,a) }{ (\set{b},a) \in \R_a }$}
For this kind of frameworks, the respective notions of \mbox{conflict-free} (resp. admissible, complete, preferred or stable) coincide with those being defined by~\citeN{Nielsen2007} and~\citeN{Dung95}, respectively.

To illustrate the notions of support and prima-facie arguments,
consider the well-known Tweety example:

\begin{example}\label{ex:tweety}
Suppose we have the knowledge base that includes the following statements:
\begin{enumerate}
\item birds (normally) can fly, 
\item penguins are birds,
\item penguins cannot fly and
\item Tweety is a penguin.
\end{enumerate}
We can formalize this by the following graph:
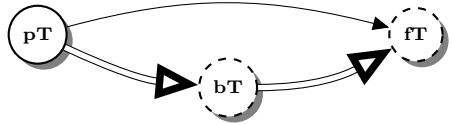
\begin{center}
\begin{tikzpicture}[tikzpict]
    \matrix[row sep=0.35cm,column sep=2cm,ampersand replacement=\&] {
      \node (a) [arg] {$\mathbf{pT}$};\&
      \node (pb) [videphantom] {};\&
      \node (d)[narg] {$\mathbf{fT}$};
      \\
      \\
     };
    \node (b)[narg, below of=pb, node distance=20pt] {$\mathbf{bT}$};
    \draw [double distance=2pt,->,-open triangle 45] (a) to [out=-25,in=175] (b);
    \draw [double distance=2pt,->,-open triangle 45] (b) to [out=0,in=210] (d);
    \draw [->,-triangle 45] (a) to [out=15,in=165] (d);
\end{tikzpicture}
\end{center}
where $pT$, $bT$ and $fT$ respectively stand for ``Tweety is a penguin'', ``Tweety is a bird''
and ``Tweety can fly.''
Double arrows represent support while simple ones represent attacks.
Furthermore, circles with solid border represent prima-facie arguments while dashed border
ones represent ordinary ones.
That is, 
``Tweety is a penguin'' is considered a prima-facie argument that supports that 
``Tweety is a bird'' which, in its turn, supports that
``Tweety can fly.''
The latter is then considered also prima-facie, that is, true unless proven otherwise.
Note that
``Tweety is a penguin''
also attacks that
``Tweety can fly'',
so the latter cannot be accepted as true.
Formally, this corresponds to the framework
\mbox{$\newef\label{ef:tweety} \hspace{-1pt}=\hspace{-1pt}
\tuple{\A\hspace{-1pt},\hspace{-1pt}\R_a\hspace{-1pt},\hspace{-1pt}\R_s\hspace{-1pt},\PF}$}
with
\mbox{$\R_a = \set{ (\set{pT},fT) }$}
and
\mbox{$\R_s = \set{ (\set{pT},bT), \, (\set{bT},fT) }$}
and
\mbox{$\PF = \set{ pT }$}
whose unique admissible, complete, preferred and stable extension is
\mbox{$\set{pT , bT}$}.
In other words, we conclude that ``Tweety cannot fly.''
Note that ``Tweety is a penguin'' provides conflicting evidence for whether it can fly or not.
In EBAFs, this is solved by giving priority to the attack relation, so ``Tweety cannot fly''
is inferred.
\end{example}

\section{Reasoning with Contradictory Evidence in Equilibrium Logic}

In this section, we formalize principles~\ref{p:contradictory} and~\ref{p:cwa} in constructive logic,
obtaining as a result a formalism which is a conservative extension of logic programming  under the answer set semantics (see Theorem~\ref{thm:conservative.ht} and Corollary~\ref{cor:conservative.asp} below)
and which is capable of reasoning with contradictory evidence.
We start by defining a new implication connective that captures~\ref{p:contradictory} in terms of intuitionistic implication and strong negation:
\begin{gather}
\varphi_1 \sup \varphi_2 \ \ \eqdef \ \ (\neg\!\sneg\varphi_1 \wedge \varphi_1) \to  \varphi_2
\label{eq:sup.def}
\end{gather}
Recall that intuitionistic implication \mbox{$\varphi_1 \to \varphi_2$} can be informally understood as a means to construct a proof of the truth of the consequent~$\varphi_2$ in terms of a proof of truth of the antecedent~$\varphi_1$.
In this sense,~\eqref{eq:sup.def}
can be understood as a means to construct a proof of the truth of the consequent~$\varphi_2$ in terms of proof of the truth of the antecedent~$\varphi_1$ and the absence of a proof of its falsity, or in other words,
in terms of a \emph{consistent proof} of the antecedent~$\varphi_1$.
It is easy to see that~\eqref{eq:sup.def} is weaker than intuitionistic implication, that is, that 
\begin{gather*}
\varphi_1 \to \varphi_2 \ \modelsp \ \varphi_1 \sup \varphi_2
\end{gather*}
holds for every pair of formulas $\varphi_1$ and~$\varphi_2$.
We can use the following simple example to illustrate the difference between intuitionistic implication and~\eqref{eq:sup.def}.

\begin{example}\label{ex:impl.diff}
Let $\newtheory\label{th:imp.diff}$ be the following set of formulas:
\begin{align*}
&a
&
&b
&
\sneg&\hspace{1.5pt}b
&
a &\sup c
&
b &\sup d
\end{align*}
and let $\theory\ref{th:imp.diff}'$ be the theory obtained by replacing each occurrence of implication~$\sup$ by intuitionistic implication~$\to$.
On the one hand, we have that both, $\theory\ref{th:imp.diff}$ and $\theory\ref{th:imp.diff}'$,
entail atoms $a$ and $c$.
On the other hand,
we have:
\mbox{$\theory\ref{th:imp.diff}' \modelsp d$}
but
\mbox{$\theory\ref{th:imp.diff} \not\modelsp d$}.
This is in accordance with~\ref{p:contradictory},
since the only way to obtain a proof of $d$ is in terms of~$b$, for which we have contradictory evidence.
Note also that an alternative proof of $d$ could be obtained if new consistent evidence becomes available:
for the theory
\mbox{$\newtheory\label{th:imp.diff2} = \theory\ref{th:imp.diff} \cup \set{ a \sup d}$}
we obtain
\mbox{$\theory\ref{th:imp.diff2} \modelsp d$}.
It is also worth highlighting that, in contrast with intuitionistic implication, this new connective~\eqref{eq:sup.def} is not monotonic: for
\mbox{$\newtheory\label{th:imp.diff3} = \set{ b ,\ b \sup d}$}
we have 
\mbox{$\theory\ref{th:imp.diff3} \modelsp d$}
and $\theory\ref{th:imp.diff3} \cup \set{ \sneg b} \not\modelsp d$.
Obviously, it is not antimonotonic either: $\theory\ref{th:imp.diff3} \setminus \set{ b} \not\modelsp d$.
\end{example}

The following result shows that, when dealing with consistent evidence, these differences disappear and~\eqref{eq:sup.def} collapses into intuitionistic implication:

\begin{proposition}{\label{prop:imp.consitent}}
Let $\sI$ be a consistent \Ninterpretation\ and let $\varphi_1$ and~$\varphi_2$ be any pair of formulas.
Then, $\sI \modelsp \varphi_1 \sup \varphi_2$ \ iff \ $\sI \modelsp \varphi_1 \to \varphi_2$.
\end{proposition}

Let us now formalize the~\ref{p:cwa} assumption.
As usual \mbox{non-monotonicity} is obtained by considering equilibrium models~(Definition~\ref{def:equilibrium}).
However, to capture~\ref{p:cwa}, we need to restrict the consequences of these models to those that are consistent.
We do so by introducing a new \emph{\mbox{cw\nobreakdash-inference}} relation which, precisely, restricts the consequences of~$\modelsp$ to those which are consistent:
\begin{gather}
\sI\sep w \models \varphi \quad\text{iff}\quad
\sI\sep w \modelsp \neg\!\sneg \varphi \wedge \varphi
\label{eq:val.def}
\end{gather}
Furthermore, as usual, we write
\mbox{$\sI \models \varphi$}
iff 
\mbox{$\sI\sep w \models \varphi$}
for all
\mbox{$w \in W$}.
We also write
\mbox{$\Gamma \models \varphi$}
iff 
\mbox{$\sI \models \varphi$}
holds for every equilibrium model~$\sI$ of $\Gamma$.
For instance, in Example~\ref{ex:impl.diff}, it is easy to see that
\mbox{$\theory\ref{th:imp.diff} \modelsp b$}
and
\mbox{$\theory\ref{th:imp.diff} \modelsp \sneg b$},
but
\mbox{$\theory\ref{th:imp.diff} \not\models b$}
and 
\mbox{$\theory\ref{th:imp.diff} \not\models \sneg b$}
because the unique equilibrium model of
\mbox{$\theory\ref{th:imp.diff}$}
contains contradictory evidence for~$b$.
On the other hand, as may be expected, when we deal with non-contradictory evidence cw\nobreakdash-inference~$\models$ just collapses to the regular inference relation~$\modelsp$
(see Proposition~\ref{prop:consitent} below).

To finalize the formalization of~\ref{p:cwa},
we also need to define
\emph{default negation}.
This is accomplished by introducing a new connective $\Not$ and adding the following two items to the Nelson's forcing relations:
\begin{IEEEeqnarray*}{l ,C, l}
\sI\sep w \modelsp \Not \varphi &\text{ iff }& \sI\sep w \modelsp \neg \varphi \vee (\varphi \wedge \sneg\varphi)
\label{eq:neg.def}
\\
\sI\sep w \modelsn \Not \varphi &\text{ iff }&  \sI\sep w \modelsp \varphi \text{ and } \sI\sep w \not\modelsn \varphi
\end{IEEEeqnarray*}
Then, an \emph{\mbox{extended formula}}~$\fF$ is defined using the following grammar:
\[
\fF \quad ::= \quad \bot \ \mid \ a \ \mid \ \sneg \fF \ \mid \ \Not \varphi \ \mid \ \fF \wedge \fF \ \mid \ \fF \vee \fF \ \mid \ \fF \to \fF
\]
with $a \in \at$ an atom.
The following result shows that \mbox{cw\nobreakdash-inference} and default negation are conservative extensions of the satisfaction relation~$\modelsp$ and HT\nobreakdash-negation~$\neg$
when restricted to consistent knowledge.

\begin{proposition}{\label{prop:consitent}}
Let $\sI$ be a consistent \Ninterpretation\ and $\varphi$ be any extended formula.
Then, the following conditions hold:
\begin{enumerate}[ label=\roman*), leftmargin=20pt]

\item $\sI \models \varphi$ \ iff \ \mbox{$\sI \modelsp \varphi$}
\label{item:models:prop:consitent}

\item  $\sI \models \Not \varphi$ \ iff \ $\sI \models \neg \varphi$.
\label{item:neg:prop:consitent}
\end{enumerate}
\end{proposition}

Despite the relation between default negation~$\Not$ and HT\nobreakdash-negation~$\neg$ on consistent interpretations, in general, they no not coincide.
The following example illustrates the difference between these two kinds of negations:

\begin{example}\label{ex:negation}
Let $\newtheory\label{th:negation}$ be the following theory:
\begin{gather*}
a \hspace{1.25cm}
\sneg a \hspace{1.25cm}
\Not \sneg a \sup b
\end{gather*}
This theory has a unique equilibrium model $\sI = \tuple{\bT,\bT}$
with $\bT = \set{ a , \sneg a, b }$.
Note that, every model $\sJ$ of~$\theory\ref{th:negation}$
must satisfy \mbox{$\sJ \modelsp a \wedge \sneg a$}
and, thus, it must also satisfy $\sJ \models \Not \sneg a$
and $\sJ \modelsp b$ follows (Proposition~\ref{prop:cw-relation}).
Hence, $\sI$ is a $\leq$-minimal model and, thus, an equilibrium model.
On the other hand, let
$\newtheory\label{th:negation2}$ be the theory:
\begin{gather*}
a \hspace{1.25cm}
\sneg a \hspace{1.25cm}
\neg\!\sneg a \sup b
\end{gather*}
In this case, we can check that
\mbox{$\sJ = \tuple{\bH,\bT}$}
with
\mbox{$\bH = \set{ a, \sneg a}$}
is a model of~$\theory\ref{th:negation2}$
because
\mbox{$\sJ \not\models \neg\!\sneg a$}
and, thus, now $\sI$ is not an equilibrium model.
In fact, $\tuple{\bH,\bH}$ is the unique equilibrium model of~$\theory\ref{th:negation2}$.
\end{example}

The following result shows the relation between default negation, implication and \mbox{cw\nobreakdash-inference}.

\begin{proposition}{\label{prop:cw-relation}}
Let $\sI$ be any \Ninterpretation\ and $\varphi$ be any formula.
Then, 
\begin{enumerate}[ label=\roman*), leftmargin=20pt]
\item $\sI \models \varphi$ and $\sI \modelsp \varphi \sup \psi$ implies \mbox{$\sI \modelsp \psi$},
\label{item:impl:prop:cw-relation}

\item $\sI \models \Not \varphi$ implies $\sI \not\models \varphi$.
\label{item:neg:1:prop:cw-relation}
\end{enumerate}
Furthermore, if $\sI$ is a total \HTinterpretation, then
\begin{enumerate}[ label=\roman*), leftmargin=20pt, start=3]
\item $\sI \models \Not \varphi$ iff $\sI \not\models \varphi$.
\label{item:neg:prop:cw-relation}
\end{enumerate}
\end{proposition}

\noindent
Condition~\ref{item:impl:prop:cw-relation} formalizes a kind of \emph{modus ponens} for $\sup$ in the sense that, if the we have a consistent proof of the antecedent, then we have a (possibly inconsistent) proof of the consequent.
It is clear that this statement cannot be strengthened to provide a consistent proof of the consequent because any other formula could provide the contradictory evidence to make it inconsistent.
Note also that this relation is non-monotonic as adding new information may result in a contradictory antecedent.
Condition~\ref{item:neg:prop:cw-relation} formalizes the~\mbox{\ref{p:cwa}} assumption, that is, $\Not \varphi$ holds whenever $\varphi$ is not known to be true or we have contradictory evidence for it.
Note that, according to this, the default negation of an inconsistent formula is true and, therefore, the evaluation of default negation itself is always consistent (even if the formula is inconsistent):
that is, \mbox{$\sI\sep w \not\modelsp \Not \varphi$}
or
\mbox{$\sI\sep w \not\modelsn \Not \varphi$}
holds for any extended formula.

On the contrary that implication~$\sup$, default negation~$\Not$ cannot be straightforwardly defined\footnote{It is still an open question whether it is definable in terms of Nelson's connectives or not.} in terms of Nelson's connectives.

Another alternative, we have investigated was defining $\Not \varphi$ as
and
$\neg \varphi \vee (\varphi \wedge \sneg\varphi)$.
in terms of \mbox{cw\nobreakdash-inference}.
The following result shades light on this attempt.

\begin{proposition}{\label{prop:default.negation.alt}}
Let $\sI$ be any \Ninterpretation\ and $\varphi$ be any formula.
Then, $\sI \models \neg \varphi \vee (\varphi \wedge \sneg\varphi)$ iff
$\sI \models \neg \varphi$.

\end{proposition}

\noindent
That is, in terms of \mbox{cw\nobreakdash-inference}, $\neg \varphi \vee (\varphi \wedge \sneg\varphi)$ is equivalent to HT\nobreakdash-negation.
As illustrated by Example~\ref{ex:negation}, default negation and HT\nobreakdash-negation do not behave in the same way.

The following example illustrates that, though default negation allows to derive new knowledge from contradictory information, it does not allow to self justify a contradiction.

\begin{example}
Let $\newtheory\label{th:default}$ be a logic program containing the following single rule:
\begin{gather}
\Not \sneg a \sup a 
  \label{eq:default.rule}
\end{gather}
stating, as usual, that $a$ holds by default.
As expected this theory has a unique equilibrium model~$\sI$ which satisfies
\mbox{$\sI \models a$}
and
\mbox{$\sI \not\models \sneg a$}.
Let now
\mbox{$\newtheory\label{th:default2} = \theory\ref{th:default} \cup \set{ \sneg a}$}.
This second theory also has
a unique equilibrium model $\sI$ which now satisfies
\mbox{$\sI \models \sneg a$}
and
\mbox{$\sI \not\models a$}.
To see that
\mbox{$\sJ = \tuple{\bT,\bT}$} 
with
\mbox{$\bT = \set{ a , \sneg a }$}
is not an equilibrium model of~\theory\ref{th:default2},
let
\mbox{$\sJ' = \tuple{\bH,\bT}$} 
with
\mbox{$\bH = \set{ \sneg a }$}
be an interpretation.
Since $\sJ'$ satisfies \mbox{$\sJ' < \sJ$} and it is a model of $\sneg a$,
it only remains to be shown that
$\sJ'$ is a model of~\eqref{eq:default.rule}.
For that, just note
\mbox{$\sJ \models \sneg a$}
and, thus,
\mbox{$\sJ \not\models \Not \sneg a$}
follows by Proposition~\ref{prop:cw-relation}.
This implies that
$\sJ'$ satisfies~\eqref{eq:default.rule}
and, consequently, that
$\sJ$ is not an equilibrium model.
In fact, $\tuple{\bH,\bH}$ is the unique equilibrium model of~$\theory\ref{th:default2}$.
\end{example}

\subsection{A Conservative Extension of Logic Programming}

Let us now consider the language formed with the set of logical connectives
$$\LanLP \eqdef \set{ \bot, \sneg\hspace{1.5pt}, \wedge, \vee, \sup, \Not }$$
In other words, a \emph{\mbox{$\LanLP$-formula}}~$\fF$ is defined using the following grammar:
\[
\fF \quad ::= \quad \bot \ \mid \ a \ \mid \ \sneg \fF \ \mid \ \Not \varphi \ \mid \ \fF \wedge \fF \ \mid \ \fF \vee \fF \ \mid \ \fF \sup \fF
\]
with \mbox{$a \in \at$} being an atom.
A $\LanLP$-literal is either an explicit literal~$l$ or is default negation~$\Not l$.
A $\LanLP$-rule is a formula of the form \mbox{$H \supl B$}
where $H$ is a disjunction of atoms and $B$ is a conjunction of \mbox{$\LanLP$-literals}.
\emph{\mbox{$\LanLP$-theories}} and \emph{\mbox{$\LanLP$-programs}} are respectively defined as sets of \mbox{$\LanLP$-formulas} and \mbox{$\LanLP$-rules}.
The definition of an answer set is applied straightforwardly as in Definition~\ref{def:answer.set}.
Given any theory $\LanLP$-theory $\Gamma$, 
by $\LN{\Gamma}$ we denote the result of 
\begin{enumerate}[ leftmargin=15pt ]
\item replacing every occurrence of $\sup$ by~$\to$ and
\item and every occurrence of $\Not$ by $\neg$.
\end{enumerate}
Then, the following results follow directly from Propositions~\ref{prop:imp.consitent} and~\ref{prop:consitent}:
\black

\begin{theorem}\label{thm:conservative.ht}
Let $\Gamma$ be any \emph{$\LanLP$-theory} and $\sI$ be any consistent interpretation.
Then, $\sI$ is an equilibrium model of $\Gamma$ iff $\sI$ is an equilibrium model of $\LN{\Gamma}$.
\end{theorem}

\begin{corollary}\label{cor:conservative.asp}
Let $P$ be a $\LanLP$-program and $\bT$ be any consistent set of explicit literals.
Then, $\sI = \tuple{\bT,\bT}$ is an equilibrium model of $P$ iff $\bT$ is an answer set of~$P$.
\end{corollary}

In other words, the equilibrium models semantics are a conservative extension of the answer set semantics.
The following example shows the usual representation of the Tweety scenario in this logic (an alternative representation using contradictory evidence will be discussed in the Discussion section).

\begin{examplecont}{ex:tweety}\label{ex:tweety2}
Consider again the Tweety scenario.
The following logic program~$\newprogram\label{prg:tweety}$ is a usual way of representing this scenario in LP:
\begin{IEEEeqnarray}{rlCl}
&\mathit{flyTweety}  &\supl& \mathit{birdTweety} \wedge \Not \sneg\mathit{flyTweety}
  \label{eq:rule.flyTweety}
\\
&\mathit{birdTweety} &\supl&  \mathit{penguinTweety}
\\
\sneg\,&\mathit{flyTweety} &\supl& \mathit{penguinTweety} 
\\
&\IEEEeqnarraymulticol{2}{l}{\mathit{penguinTweety}} \notag
\end{IEEEeqnarray}
where rule~\eqref{eq:rule.flyTweety} formalizes the statement~``birds normally can fly.''
This is achieved by considering $\sneg\mathit{flyTweety}$ as an exception to this rule.
It can be checked that $\program\ref{prg:tweety}$ has a unique equilibrium model~$\sI_{\ref{prg:tweety}}$, which is consistent, and which satisfies~$\sI_{\ref{prg:tweety}} \not\models \mathit{flyTweety}$
and~$\sI_{\ref{prg:tweety}} \models \Not\mathit{flyTweety}$.
In other words, Tweety cannot fly.
\end{examplecont}

\begin{examplecont}{ex:impl.diff}\label{ex:impl.diff.neg}
Consider now the theory obtained by replacing formulas $a \sup c$ and $b \sup d$
in~$\theory\ref{th:imp.diff}$ by the following two formulas:
\begin{IEEEeqnarray*}{rCl " rCl}
\Not e \wedge a &\sup& c
&
\Not e \wedge b &\sup& d
\end{IEEEeqnarray*}
Let $\newtheory\label{th:imp.diff.neg}$ be such theory.
It is easy to see that 
neither
$\theory\ref{th:imp.diff.neg}$
nor
$\LN{\theory\ref{th:imp.diff.neg}}$
monotonically entail~$c$ nor~$d$.
This is due to the fact that the negation of $e$ is not monotonically entailed:
\mbox{$\theory\ref{th:imp.diff.neg} \not\modelsp \Not e$}
and
\mbox{$\LN{\theory\ref{th:imp.diff.neg}} \not\modelsp \neg e$}.
On the other hand, the negation of $e$ is \mbox{non-monotonically} entailed in both cases:
\mbox{$\theory\ref{th:imp.diff.neg} \models \Not e$}
and
\mbox{$\LN{\theory\ref{th:imp.diff.neg}} \models \neg e$}.
Note that
both \mbox{$\theory\ref{th:imp.diff.neg}$} and
\mbox{$\LN{\theory\ref{th:imp.diff.neg}}$}
have a unique equilibrium model,
\mbox{$\sI_{\ref{th:imp.diff.neg}} = \tuple{\bT,\bT}$}
and
\mbox{$\sI_{\ref{th:imp.diff.neg}}' = \tuple{\bT',\bT'}$}
with $\bT = \set{a,b ,\sneg b, c}$
and
\mbox{$\bT' = \set{a,b, \sneg b, c, d}$}, respectively,
and in both cases we have
\mbox{$\sI_{\ref{th:imp.diff.neg}} \models \Not e$} and
\mbox{$\sI_{\ref{th:imp.diff.neg}}' \models \neg e$}.
As a result, we get that both theories cautiously entail~$c$.
However, as happened in Example~\ref{ex:impl.diff}, only $\LN{\theory\ref{th:imp.diff.neg}}$ cautiously entails~$d$, because the unique evidence for $d$ comes from~$b$ for which we have inconsistent evidence.
This behavior is different from paraconsistent answer sets~\cite{SakamaI95,OdintsovP05}.
As pointed out by~\citeN{SakamaI95}, the truth of~$d$ is less credible than the truth of $c$, since $d$ is derived through the contradictory fact~$b$.
In order to distinguish such two facts~\citeN{SakamaI95} also define \emph{suspicious answer sets}
which do not consider $d$ as true.\footnote{
Suspicious answer sets are based on a 6-value lattice which add the values \emph{suspiciously true} and \emph{suspiciously false} to the four values of~$\Nn$.
In the unique suspicious answer set of
$\theory\ref{th:imp.diff.neg}$,
atom~$d$ gets assigned the suspiciously true value instead the true value.
A formal comparison with suspicious answer sets is left for future work.}

This example also helps us to illustrate the strengthened closed world assumption principle~\ref{p:cwa}.
On the one hand, we have that
\mbox{$\theory\ref{th:imp.diff.neg} \models \Not e$}
holds because there is no evidence for $e$.
On the other hand, 
we have that
\mbox{$\theory\ref{th:imp.diff.neg} \models \Not b$}
holds because we have contradictory evidence for $b$.
Moreover, we have that \mbox{$\theory\ref{th:imp.diff.neg} \models \Not d$} holds
because the only evidence we have for $d$ is based on the contradictory evidence for $b$.
\end{examplecont}

\section{Argumentation Frameworks in Equilibrium Logic}

In this section, we show how AFs, SETAFs and EBAFs can be translated in this logic in a modular way and using only the object language.
This translation is a formalization of the intuition of an attack stated in~\ref{p:attack}.
Theorems~\ref{thm:af.stable<->emodel}, \ref{thm:sf.stable<->emodel} and~\ref{thm:ef.stable<->emodel} show that the equilibrium models of this translation precisely characterize the stable extension of the corresponding framework.

\subsection{Dung's Argumentation Frameworks}

Now, let us formalize the notion of attack introduced in~\ref{p:attack}, by defining the following connective:
\begin{gather}
\varphi_1 \attacks \varphi_2 \ \ \eqdef \ \ \varphi_1 \sup \sneg\varphi_2
\label{eq:att.def}
\end{gather}
Here we identify the acceptability of $\varphi_1$ with having a consistent proof of it, or in other words, as having a proof of the truth of $\varphi_1$ and not having a proof of its falsity.
Then, \eqref{eq:att.def} states that the acceptability of $\varphi_1$ allows to construct a proof of the falsity of $\varphi_2$.
In this sense, we identify a proof of the falsity of~$\varphi_2$
with $\varphi_2$ being defeated.

\begin{proposition}\label{prop:att.mponens}
  Given any \Ninterpretation~$\sI$ and any pair of formulas $\varphi_1,\varphi_2$, the following conditions hold:
  \begin{enumerate}[ label=\roman*)]
  
  \item $\sI \models \varphi_1$ and $\sI \modelsp \varphi_1 \att \varphi_2$
  imply $\sI \modelsn \varphi_2$
  \label{item:2:prop:sup+att.models}
  \end{enumerate}
\end{proposition}
Using the language $\LanAF = \set{\att}$, we can translate any AF as follows:

\begin{definition}\label{def:af.translation}
Given some framework $\AFF$, we define the theory:
\begin{IEEEeqnarray}{ l ,C, l}
\LAF{\AF} &\eqdef& \A \cup \setm{ a \att b }{ (a,b) \in \R }
  \label{eq:def.gamma.af}
\end{IEEEeqnarray}
In addition, we assign a corresponding
set of arguments \mbox{$\SI \eqdef \setm{\! a \in \A \!}{\! \sI \models a \!}$}
to every interpretation~$\sI$.
\end{definition}

Translation~$\LAF{\cdot}$ applies the notion of attack introduced in~\ref{p:attack} to translate an AF into a logical theory.
The strengthened close world assumption~\ref{p:cwa} is used to retrieve the arguments~$\SI$ corresponding to each stable model~$\sI$ of the logical theory obtained from this translation.

\begin{example}\label{ex:af.line}
To illustrate this translation, let $\newaf\label{af:line}$ be the framework corresponding to the following graph:
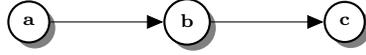
\begin{center}
\begin{tikzpicture}[tikzpict]
    \matrix[row sep=0.5cm,column sep=1.5cm,ampersand replacement=\&] {
      \node (a) [arg] {$\mathbf{a}$};\&
      \node (b)[arg] {$\mathbf{b}$};\&
      \node (c)[arg] {$\mathbf{c}$};\\
     };
    \draw [->,-triangle 45] (a) to (b);
    \draw [->,-triangle 45] (b) to (c);
\end{tikzpicture}
\end{center}
Then, we have that $\LAF{\af\ref{af:line}}$ is the theory containing the following two attacks:
\begin{gather*}
a \att b \hspace{2cm} b \att c
\end{gather*}
plus the facts $\set{a,b,c}$.
\end{example}

\begin{proposition}{\label{prop:af.model}}
Let \mbox{$\AF$} be some framework and
\mbox{$\sI$} be some \HTmodel\ of $\LAF{\AF}$.
Then, the following hold:
\begin{enumerate}[ label=\roman*), leftmargin=20pt]

\item if $a$ is defeated \wrt~$\SI$, then $\sI \modelsp \sneg a$
\label{item:1:prop:af.model}

\item $\SI$ is \cfree.
\end{enumerate}
If, in addition, $\sI$ is an $\leq$-minimal model, then
\black
\begin{enumerate}[ label=\roman*), start=3, leftmargin=20pt]

\item $a$ is defeated \wrt~$\SI$ iff $\sI \modelsp \sneg a$.
\label{item:3:prop:af.model}
\end{enumerate}
\end{proposition}

\begin{examplecont}{ex:af.line}
Continuing with our running example,
let
\mbox{$\sI_{\ref{af:line}} = \tuple{\bT_{\ref{af:line}},\bT_{\ref{af:line}}}$}
and
\mbox{$\sJ_{\ref{af:line}} = \tuple{\bT_{\ref{af:line}}',\bT_{\ref{af:line}}'}$}
be two total models of $\Gamma_{\af\ref{af:line}}$
with
\mbox{$\bT_{\ref{af:line}} = \set{a,b,c,\sneg b}$}
and
\mbox{$\bT_{\ref{af:line}}' = \set{a,b,c,\sneg a,\sneg c}$}.
Then, we have that both
\mbox{$S_{\sI_{\ref{af:line}}} = \set{a,c}$}
and
\mbox{$S_{\sJ_{\ref{af:line}}} = \set{b}$} are \mbox{conflict-free}
(though only $S_{\sI_{\ref{af:line}}}$ is stable).
Furthermore, we also can see that argument~$b$ is the unique defeated argument \wrt~$S_{\sI_{\ref{af:line}}}$ and the unique atom for which $\sI_{\ref{af:line}} \modelsp \sneg b$ holds.
On the other hand, we get that argument~$c$ is the unique defeated argument \wrt~$S_{\sJ_{\ref{af:line}}}$ and also
both
\mbox{$\sJ_{\ref{af:line}} \modelsp \sneg a$}
and
\mbox{$\sJ_{\ref{af:line}} \modelsp \sneg c$}
hold.
Note that, as stated by~\ref{item:3:prop:af.model} in Proposition~\ref{prop:af.model},
this implies that only $S_{\sI_{\ref{af:line}}}$ can be an equilibrium model.
Let us show that it is indeed the case that $\sJ_{\ref{af:line}}$ is not an equilibrium model
and let us define, for that purpose, an interpretation~$\sJ_{\ref{af:line}}' = \tuple{\bH_{\ref{af:line}}',\bT_{\ref{af:line}}'}$
with \mbox{$\bH_{\ref{af:line}}' = \bT_{\ref{af:line}}' \setminus \set{ \sneg a} = \set{a,b,c,\sneg c}$}.
In other words, interpretation~$\sJ_{\ref{af:line}}'$ is as $\sJ_{\ref{af:line}}$, but removing the non-defeated argument~$a$ as a negated conclusion~$\sneg a$.
It is easy to check that
\mbox{$\sJ_{\ref{af:line}}' \models b \att c$} because \mbox{$\sneg c \in \bH_{\ref{af:line}}'$} holds.
Besides, since \mbox{$\sneg a \in \bT_{\ref{af:line}}'$}, we have that $\sJ_{\ref{af:line}}' \not\models a$ and, therefore, that $\sJ_{\ref{af:line}}' \models a \att b$.
This implies that $\sJ_{\ref{af:line}}'$ is a model of~$\Gamma_{\af\ref{af:line}}$.
Since $\sJ_{\ref{af:line}}' <  \sJ_{\ref{af:line}}$, we get that $\sJ_{\ref{af:line}}$ is not an equilibrium model.
\end{examplecont}

In fact, we can generalize this correspondence between the stable extensions and the equilibrium models to any argumentation framework as stated by the following theorem:

\begin{theorem}\label{thm:af.stable<->emodel}
Given some $\AFF$, there is a one-to-one correspondence between its stable extensions and the equilibrium models of $\LAF{\AF}$ such that
\begin{enumerate}[ label=\roman*)]
\item if $\sI$ is an equilibrium model of $\LAF{\AF}$, then $\SI$ is a stable extension of~$\AF$,\label{item:1:thm:stable<->emodel}

\item if $\aA$ is a stable extension of $\AF$ and $\sI$ is a total interpretation such that $T_\sI^+ \!=\! \A$ and $T_\sI^- \!=\! \Defeated{\aA}$, then $\sI$ is an equilibrium model of $\LAF{\AF}$.
\label{item:2:thm:stable<->emodel}
\end{enumerate}
\end{theorem}

\begin{proofs}\footnote{This theorem is a particualr case of Theorem~\ref{thm:sf.stable<->emodel} below. Recall that full proofs are provided in the appendix. }
First, note that condition~\ref{item:1:thm:stable<->emodel}
follows directly from~\ref{item:3:prop:af.model} in Proposition~\ref{prop:af.model}
and the facts that $(a)$ equilibrium models are $\leq$-minimal models and $(b)$
$\SI$ is a stable extension iff
$\SI$ are exactly the non-defeated arguments \wrt~$\SI$.
To show~\ref{item:2:thm:stable<->emodel},
it is easy to see that~$\SI$ being a stable extension implies that $\sI$ is a 
model of~$\LAF{\AF}$.
Hence, to show that $\sI$ is an equilibrium model what remains is to prove that any $\sJ < \sI$ is not a model of $\LAF{\AF}$.
Any such $\sJ$ must satisfy 
\mbox{$H_\sJ^+ = H_\sI^+ = \A$}
and
\mbox{$H_\sJ^- \subset H_\sI^- = T_\sI^- = \Defeated{\aA}$}.
Therefore, there is some defeated argument such that \mbox{$a \notin H_\sJ^-$}
and some defeating attack \mbox{$(b,a) \in \R_a$} such that
\mbox{$b \in \aA = H_\sI^+ \setminus T_\sI^- =  H_\sJ^+ \setminus T_\sJ^-$}.
This implies that \mbox{$b \att a \in \LAF{\AF}$}
and $\sJ \models b$
which, in its turn, implies that
$a \in H_\sJ^-$.
This is a contradiction
and, consequently, $\sI$ is an equilibrium model.
\end{proofs}

Theorem~\ref{thm:af.stable<->emodel} captures the relation between the stable extensions of an~AF and its translation into a logical theory.
As mentioned above, this relation relies on the reasoning principles~\ref{p:attack} and~\ref{p:cwa}:
An~$\AFF$ is translated into a logical theory~$\LAF{\AF}$ using the notion of attack introduced in~\ref{p:attack}.
The stable extension~$\SI$ of this~AF is then retrieved from the equilibrium model~$\sI$ of $\LAF{\AF}$ using the~\ref{p:cwa} principle.

\subsection{Set Attack Argumentation Frameworks}

We may also extend the results of the previous section to SETAFs using the language $\LanSF = \set{\att, \wedge}$ 
and a similar translation.

\begin{definition}
Given some finitary set attack framework \mbox{$\SFF$}, we define 
\begin{IEEEeqnarray}{ l ,C, l}
\Gamma_{\R_a} &\eqdef& \setBm{ \bigwedge A \att b }{ (A,b) \in \R_a}
  \label{eq:def.gamma.sf}
\end{IEEEeqnarray}
and \mbox{$\LSF{\SF} \eqdef \A \cup \Gamma_{\R_a}$}.
\end{definition}

Similar to Definition~\ref{def:af.translation},
translation~$\LSF{\cdot}$ applies the notion of attack introduced in~\ref{p:attack} to translate an AF into a logical theory.
In this case the set of attacking arguments becomes a conjuntion in the antecedent of the attack connective.

\begin{theorem}{\label{thm:sf.stable<->emodel}}
Given some finitary $\SF$ there is a one-to-one correspondence between its stable extensions and the equilibrium models of $\LSF{\SF}$ such that
\begin{enumerate}[ label=\roman*)]
\item if $\sI$ is an equilibrium model of $\LSF{\SF}$, then $\SI$ is a stable extension of $\SF$,\label{item:1:thm:sf.stable<->emodel}

\item if $\aA$ is a stable extension of $\SF$ and $\sI$ is a total interpretation such that $T_\sI^+ \!=\! \A$ and  $T_\sI^- \!=\! \Defeated{\aA}$, then $\sI$ is an equilibrium model of $\LSF{\SF}$.
\label{item:2:thm:sf.stable<->emodel}
\end{enumerate}
\end{theorem}

\begin{proofs}
The proof follows as in Theorem~\ref{thm:af.stable<->emodel} by noting that any interpretation~$\sI$ and set of arguments $B$ satisfy: $B \subseteq \SI$ iff 
$\sI \models b$ for all $b \in B$ iff $\sI \models \bigwedge B$.
\end{proofs}

\subsection{Argumentation Frameworks with Evidence-Based Support}

Let us now extend the language of SETAFs with the LP~implication~\eqref{eq:sup.def}, in other words,
we consider the language possessing the following set of connectives
\mbox{$\LanEF = \set{\att,\wedge,\sup}$}, so that we can translate any EBAF
as follows:

\begin{definition}
Given any finitary evidence-based framework~\mbox{$\EBAFF$}, 
we define its corresponding theory as: \mbox{$\LEF{\EBAF} \,\eqdef\, \PF \cup \Gamma_{\R_a} \cup \Gamma_{\R_s}$} with
\begin{IEEEeqnarray}{ l ,C, l}
\Gamma_{\R_s} &\eqdef& \setBm{ \bigwedge A \sup b }{ (A,b) \in \R_s}
  \label{eq:def.gamma.sup}
\end{IEEEeqnarray}
and $\Gamma_{\R_a}$ as stated in~\eqref{eq:def.gamma.sf}.
\end{definition}

Note that, in contrast with AFs and SETAFs, the theory corresponding to an EBAFs do not contain all arguments as atoms, but only those that are \mbox{prima-facie}~$\PF$.
This reflects the fact that in EBAFs not all arguments can be accepted, but only those that are prima-facie or are supported by those prima-facie.
Supports are represented using the LP implication $\sup$ and supported arguments are captured by the positive evaluation of each interpretation~$H_\sI^+$.
The following result extends Proposition~\ref{prop:af.model} to EBAFs including the relation between supported arguments and models.

\begin{proposition}{\label{prop:ef.model}}
Let \mbox{$\EBAF$} be some framework and
\mbox{$\sI$} be some \HTmodel\ of $\LEF{\EBAF}$.
Then, the following hold:
\begin{enumerate}[ label=\roman*), leftmargin=20pt]
\item if $a$ is supported \wrt~$\SI$, then $\sI \modelsp a$,
\label{item:1:prop:ef.model}

\item if $a$ is defeated \wrt~$\SI$, then $\sI \modelsp \sneg a$,
\label{item:2:prop:ef.model}

\item $\SI$ is \cfree.
\label{item:3:prop:ef.model}
\end{enumerate}
If, in addition, $\sI$ is an $\leq$-minimal \HTmodel, then
\begin{enumerate}[ label=\roman*), start=3, leftmargin=20pt]
\item $a$ is supported \wrt~$\SI$ iff $\sI \modelsp a$,
\label{item:4:prop:ef.model}

\item $a$ is defeated \wrt~$\SI$ iff $\sI \modelsp \sneg a$,
\label{item:5:prop:ef.model}

\item $\SI$ is \ssupporting.
\label{item:6:prop:ef.model}
\end{enumerate}
\end{proposition}

\begin{examplecont}{ex:tweety}\label{ex:tweety3}
Consider now framework~$\EBAF$ representing the Tweety scenario.
\refstepcounter{programcount}\label{th:tweety}
\begin{IEEEeqnarray}{lCl+x*}
\mathit{birdTweety} &\sup& \mathit{flyTweety}
  \label{eq:flyTweety}
\\
\mathit{penguinTweety} &\sup& \mathit{birdTweety}
\\
\mathit{penguinTweety} &\att& \mathit{flyTweety}
 \label{eq:flyTweety.att}
\\
&&\mathit{penguinTweety} \notag&
\end{IEEEeqnarray}
\end{examplecont}

As mentioned in Example~\ref{ef:tweety},
framework~\ef\ref{ef:tweety} has a unique stable extension
$$\set{\mathit{penguinTweety},\, \mathit{birdTweety}}$$
which does not include the argument $\mathit{flyTweety}$.
In other words, Tweety cannot fly.
Interestingly, $\LSF{\ef\ref{ef:tweety}}$ has also a unique equilibrium model
\mbox{$\sI_{\ref{th:tweety}} = \tuple{\bT_{\ref{th:tweety}},\bT_{\ref{th:tweety}}}$}
where $\bT_{\ref{th:tweety}}$ stands for the set:
\begin{gather*}
\set{ \mathit{penguinTweety},\, \mathit{birdTweety},\, \mathit{flyTweety},\, \sneg\mathit{flyTweety}}
\end{gather*}
This equilibrium model 
precisely satisfies the two arguments in that stable extension:
\mbox{$\sI_{\ref{th:tweety}} \models \mathit{penguinTweety}$}
and
\mbox{$\sI_{\ref{th:tweety}} \models \mathit{birdTweety}$}.
Note that we get \mbox{$\sI_{\ref{th:tweety}} \not\models \mathit{flyTweety}$}
from the fact that
\mbox{$\sI_{\ref{th:tweety}} \modelsp \sneg\mathit{flyTweety}$}.
In fact,
this correspondence holds for any EBAF
as shown by Theorem~\ref{thm:ef.stable<->emodel} below.
Though more technically complex, the proof of Theorem~\ref{thm:ef.stable<->emodel} is similar that those of Theorems~\ref{thm:af.stable<->emodel} and~\ref{thm:sf.stable<->emodel}.
In particular, it is necessary to prove the following relation between equilibrium models and supportable arguments:

\begin{proposition}{\label{prop:ef.model.supportable}}
Let \mbox{$\EBAF$} be some framework and
\mbox{$\sI$} be some equilibrium model of $\LEF{\EBAF}$.
Then, the following statement holds:
\begin{enumerate}[ label=\roman*), start=3, leftmargin=20pt, start=1]
\item $a$ is supportable \wrt~$\SI$ iff $\sI \modelsp a$.
\label{item:6:prop:ef.model.supportable}
\end{enumerate}
\end{proposition}

In contrast with the results for supported arguments stated in Proposition~\ref{prop:ef.model},
this property does not hold for arbitrary $\leq$-minimal models.
This fact can be illustrated by considering a simple $\newef\label{ef:equilibrium}$ such that
\mbox{$\LEF{\ef\ref{ef:equilibrium}} = \set{ a,\, a \sup b }$}.
Let
$\sI_{\ref{ef:equilibrium}} = \tuple{\bH_{\ref{ef:equilibrium}},\bT_{\ref{ef:equilibrium}}}$
be some interpretation
with
\mbox{$\bH_{\ref{ef:equilibrium}} = \set{ a }$}
and
\mbox{$\bT_{\ref{ef:equilibrium}} = \set{ a, \sneg a }$}.
It is easy to see that $\sI_{\ref{ef:equilibrium}}$ is a $\leq$-minimal model of
$\LEF{\ef\ref{ef:equilibrium}}$, though it is not an equilibrium model (because it is not a total interpretation).
It can also be checked that
$a$ is not defeated and, consequently, that $b$ is supportable \wrt~$\aA_{\sI_{\ref{ef:equilibrium}}} = \varnothing$.
On the other hand, the unique equilibrium model of
$\LEF{\ef\ref{ef:equilibrium}}$
is
$\sJ_{\ref{ef:equilibrium}} = \tuple{\bH'_{\ref{ef:equilibrium}},\bT'_{\ref{ef:equilibrium}}}$
with
\mbox{$\bH'_{\ref{ef:equilibrium}} = \set{ a, b }$}
and
\mbox{$\bT'_{\ref{ef:equilibrium}} = \set{ a, b }$}.
Here, both $a$ and $b$ are supportable (and supported) \wrt~$\aA_{\sJ_{\ref{ef:equilibrium}}} = \set{a,b}$.

The following result shows that, indeed, this correspondence holds for any EBAF:

\begin{theorem}{\label{thm:ef.stable<->emodel}}
Given some finitary $\EBAF$, there is a one-to-one correspondence between its stable extensions and the equilibrium models of $\LEF{\EBAF}$ such that
\begin{enumerate}[ label=\roman*)]
\item if $\sI$ is an equilibrium model of $\LEF{\EBAF}$, then $\SI$ is a stable extension of~$\EBAF$,\label{item:1:thm:ef.stable<->emodel}

\item if $\aA$ is a stable extension of $\EBAF$ and $\sI$ is a total interpretation such that $T_\sI^+ \!=\! \Supported{\aA}$ and  $T_\sI^- \!=\! \Defeated{\aA}$, then $\sI$ is an equilibrium model of~$\LEF{\EBAF}$.
\label{item:2:thm:ef.stable<->emodel}
\end{enumerate}
\end{theorem}

\section{Translation of $\LanLP$-program to regular programs}

In this section, we show how \mbox{$\LanLP$-programs} can be translated into regular ASP programs.
An important practical consequence of this fact is that current state-of-the-art ASP solvers~\cite{fapfledeie08a,gekasc09c}
can be applied to \mbox{$\LanLP$-programs}.
Let us introduce such a translation as follows:

\begin{definition}\label{def:tr1}
Given a $\LanLP$-program $P$, by $\LNp{P}$ we denote the result of 
\begin{enumerate}[ leftmargin=20pt ]
\item replacing every positive literal $a$ in the body of a rule by \mbox{$a \wedge \neg\!\sneg a$},

\item replacing every negative literal $\Not a$ in the body of a rule by \mbox{$\neg a \vee (a \!\wedge\! \sneg a)$},

\item replacing all occurrences of $\sup$ by $\to$.
\end{enumerate}
\end{definition}

\begin{proposition}\label{prop:to.N4}
Any $\LanLP$-program~$P$ and interpretation~$\sI$ satisfy:
$\sI \modelsp P$ iff $\sI \modelsp \LNp{P}$.
\end{proposition}

Proposition~\ref{prop:to.N4} shows how we can translate any $\LanLP$-program into an equivalent theory  that does not use the new connectives~$\Not$ and~$\sup$.
The result of the translation in Definition~\ref{def:tr1} is almost a standard logic program, but for two points.
First, strong negation has to be understood in a paraconsistent way, so an atom can be true and false at the same time.
This can be addressed by using new auxiliary atoms to represent strongly negated atoms.\footnote{In fact, modern solvers already allow the use of explicit negation and their implementation is done by using new auxiliary atoms to represent strongly negated atoms.
However, solvers also include a constraint of the form $a \wedge \sneg a \to \bot$ for every atom~$a$.
This would remove the non-consistent answer sets, something we have to avoid to obtain paraconsistent answer sets.
}
Second, step~2 introduces a disjunction in the body, which is not allowed in the standard syntax of logic programs.
This can be addressed in polynomial-time also by using auxiliary atoms (similar to~\citeNP{tseitin68a}).
The following definition addresses these two issues.

\begin{definition}\label{def:tr2}
Given a $\LanLP$-program $P$, by $\LNpp{P}$ we denote the result of applying the following transformations to $\LNp{P}$:
\begin{enumerate}[ leftmargin=20pt ]
\item replacing every explicit literal of the form $\sneg a$ by a fresh atom~$\tilde{a}$,

\item adding rules $a' \leftarrow \neg a$ and $a' \leftarrow a \wedge \tilde{a}$ for each atom $a \in \at$ with $a'$ a new fresh atom, and

\item replacing each occurrence of~$\neg a \vee (a \wedge \tilde{a})$ in the body of any rule by~$a'$.
\end{enumerate}
Given a total interpretation $\sI$, we also denote by $\LNpp{\sI}$ an interpretation that, for every atom~$a \in \at$, satisfies:
\begin{enumerate}[ leftmargin=20pt ]
\item $\LNpp{\sI} \not\modelsn a$ 
\item $\LNpp{\sI} \modelsp a$ iff $\sI \modelsp a$
\item $\LNpp{\sI} \modelsp \tilde{a}$ iff $\sI \modelsn a$
\item $\LNpp{\sI} \modelsp a'$ iff either $\sI \not\modelsp a$ or both $\sI \modelsp a$ and $\sI \modelsn a$.
\end{enumerate}
\end{definition}

\begin{proposition}\label{prop:tr.lp}
Any $\LanLP$-program~$P$ and total interpretation~$\sI$ satisfy that
$\sI$ is an equilibrium model of $P$ iff $\LNpp{I}$ an equilibrium model of $\LNpp{P}$.
\end{proposition}

The result of Definition~\ref{def:tr2} is a standard logic program.
Proposition~\ref{prop:tr.lp} shows that we can use this translation in combination with standard ASP solvers to obtain equilibrium for $\LanLP$-program and stable extensions of all the AFs considered in this paper.
The second consequence of this translation is that
deciding whether there exists any stable extension of some \mbox{$\LanLP$-program} is in~\mbox{$\SigmaP{2}$} in general and in~\mbox{NP} if the program is normal~\cite{daeigovo01a}.
This complexity results are tight because hardness follows from Corollary~\ref{cor:conservative.asp} and the hardness results for finding answer sets for these classes of programs~\cite{daeigovo01a}.
Therefore, deciding whether there exists any stable extension of some \mbox{$\LanLP$-program} is \mbox{$\SigmaP{2}$-complete} in general and \mbox{NP-complete} for normal \mbox{$\LanLP$-programs}.
Furthermore, this result directly applies to EBAFs so that deciding whether there exists any stable extension is \mbox{NP-complete}.

\section{Discussion}
\label{sec:dis}

LP and AFs are two \mbox{well-established} KRR formalisms for dealing with nonmonotonic reasoning (NMR).
In particular, Answer Set Programming (ASP)
is an LP paradigm, based on the stable model semantics,
which has raised as a preeminent tool for practical NMR with applications in diverse areas of AI including planning, reasoning about actions, diagnosis, abduction and beyond~\cite{baral2003knowledge,BrewkaET11}.
On the other hand, one of the major reasons for the success of AFs is their ability to handle conflicts due to inconsistent information.

Here, we have shown that both formalisms can be successfully accommodated in Nelson's constructive logic.
In fact, it is easy to see that by rewriting attacks using definition~\eqref{eq:att.def}, the translation of any AF becomes a normal $\LanLP$-program.
For instance, by rewriting the attack~\eqref{eq:flyTweety.att}, we obtain the equivalent formula:
\begin{IEEEeqnarray}{lCl+x*}
 \mathit{penguinTweety} &\sup& \sneg\mathit{flyTweety}
 \label{eq:flyTweety.att.rule}
\end{IEEEeqnarray}
which is a \mbox{$\LanLP$-rule}.
In this sense, we can consider
$\LSF{\ef\ref{ef:tweety}}$ in Example~\ref{ex:tweety3} as an alternative representation of the Tweety scenario in LP.
Note that both 
the unique equilibrium model $\sI_{\ref{prg:tweety}}$ of program~$\program\ref{prg:tweety}$ (Example~\ref{ex:tweety2})
and the unique equilibrium model $\sI_{\ref{th:tweety}}$ of this program satisfy:
\begin{IEEEeqnarray*}{l C l}
\begin{IEEEeqnarraybox}{l C l}
\sI_{\ref{prg:tweety}} \not\models \mathit{flyTweety}
\\
\sI_{\ref{prg:tweety}} \models \Not\mathit{flyTweety}
\end{IEEEeqnarraybox}
\hspace{1.25cm}
\begin{IEEEeqnarraybox}{l C l}
\sI_{\ref{th:tweety}} \not\models \mathit{flyTweety}
\\
\sI_{\ref{th:tweety}} \models \Not\mathit{flyTweety}
\end{IEEEeqnarraybox}
\end{IEEEeqnarray*}
In other words, in both programs we conclude that Tweety cannot fly.
However, there are a couple of differences between these two representations.
First, in contrast with $\sI_{\ref{prg:tweety}}$,
we have that $\sI_{\ref{th:tweety}}$ is not consistent:
\mbox{$\sI_{\ref{th:tweety}} \modelsp \mathit{flyTweety}$}
and
\mbox{$\sI_{\ref{th:tweety}} \modelsp \sneg\mathit{flyTweety}$}.
Second and perhaps more interestingly,
in $\LSF{\ef\ref{ef:tweety}}$, the ``normality'' of the statement ``birds can fly'' does not need to be explicitly represented.
Instead, this normality is implicitly handled by the strong closed word assumption~\ref{p:cwa}, which resolves the contradictory evidence for $\mathit{flyTweety}$ by regarding it as false.
In this sense, \mbox{$\LanLP$-programs} and AFs can be seen as two different syntaxes of the same formalism based 
on the principles~\ref{p:contradictory} and~\mbox{\ref{p:cwa}} highlighted in the introduction.
In addition, another principle of this formalism is the fact that evidence must be founded or justified: this clearly shows up in normal LP and EBAFs where true literals can be computed by some recursive procedure, but also in Dung's AFs where, as we have seen, defeat can be understood as a proof of falsity.

Regarding practical aspects,
we can use \mbox{$\LanLP$-programs} as a unifying formalism to deal with both logic programs and AFs.
This directly allows to introduce variables in AFs through the use of grounding.
Going further, full first\nobreakdash-order characterizations of AFs can be provided by applying the same principles to first\nobreakdash-order constructive logic (full first\nobreakdash-order characterization of consistent logic programs has been already provided by~\citeNP{PV04}).
Besides, constructive logic immediately provides an interpretation for other richer syntaxes like the use of disjunctive targets in Collective Argumentation~\cite{bochman2003collective} or the use of arbitrary propositional formulas to represent attacks in
Abstract Dialectical Frameworks~\cite{BrewkaW10,BrewkaSEWW13}.

\section{Conclusion and future work}
We have formalized the principles~\ref{p:contradictory} and~\ref{p:cwa} in Nelson's constructive logic and shown that this is a conservative extension of logic programs which allow us to reason with contradictory evidence.
Furthermore, this allows us to translate argumentation frameworks in a modular way and using the object language such that attacks and supports become connectives in logic using the object level.
As a consequence, we can combine both formalisms in an unifying one and use proof methods from the logic or answer set solver to reason about it.

Regarding future work, an obvious open topic is to explore how other argumentation semantics can be translated
into the logic.
For instance, the relation between the complete semantics for AFs, three\nobreakdash-valued stable models semantics for LP~\cite{przymusinski91a,wucaga09a} and partial equilibrium logic~\cite{caodpeva07a} suggest that our framework can be extended to cover other semantics such as the complete and preferred.
Similarly, the relation between the paracoherent semantics for AFs~\cite{americ19a} and semi\nobreakdash-equilibrium models \cite{ameifilemo16a} suggest a possible direction to capture this semantics using the object level.
It will be also interesting to see the relation with the semi\nobreakdash-stable semantics for AFs~\cite{cacadu21a}.
The relation with other AFs extensions such as Collective Argumentation~\cite{bochman2003collective}, Abstract Dialectical Frameworks~\cite{BrewkaW10,BrewkaSEWW13} or Recursive Argumentation Frameworks~\cite{BGW05-sh,Modgil09,Gabbay2009,BaroniCGG11,CCLS16b-sh,cafafala21a} is also a direction worth exploring.
Another important open questions are studying how the principles~\ref{p:contradictory} and~\ref{p:cwa} stand in the context of paraconsistent logics~\cite{dacosta74} and paraconsistent logic programming~\cite{blasub89a};
and studying the notion of strong equivalence~\cite{LifschitzPV01,oikarinen2011characterizing} in this logic and evidence-based frameworks.

\paragraph{Acknowledgements.}
We are thankful to Seiki Akama, Pedro Cabalar, Marcelo Coniglio, David Pearce, Newton Peron and Agust\'{i}n Valverde for their suggestions and comments on earlier versions of this work.
We also thank the anonymous reviewers of the Sixteenth International Conference on Principles of Knowledge Representation and Reasoning for their comments on a preliminary version of this work.

\paragraph{Competing interests:}
The authors declare none.

\bibliographystyle{acmtrans}

\appendix
\section*{Proofs of results}

\begin{proofof}{Proposition~\ref{prop:preservation}}
    First note that, if $\varphi$ is an atom, the result follows directly from the preservation of the valuations.
    Furthermore,
    the cases of $I \sep w \modelspm \varphi_1 \to \varphi_2$ with $\pm \in \set{+,-}$ follow directly by the definition.
    Otherwise, we assume as induction hypothesis that the statement holds for all subformulas of $\varphi$.
    Then, the cases of $I \sep w \modelspm \varphi_1 \otimes \varphi_2$
    with $\pm \in \set{+,-}$ and $\otimes \in \set{\wedge,\vee}$ follow directly by induction.
    The same holds for the case
    $I \sep w \modelspm \sneg\varphi$
    with  $\pm \in \set{+,-}$.
\end{proofof}

\begin{proofof}{Proposition~\ref{prop:negation}}
    For i) note that
    $I\sep w \modelsp \neg\varphi$ holds iff 
    \begin{IEEEeqnarray*}{l ?C? l }
    I\sep w \modelsp \varphi \to \bot
        &\text{ iff }& \forall w'\geq w \ I\sep w' \not\modelsp \varphi \text{ or } I\sep w' \modelsp \bot
    \\
      &\text{ iff }& \forall w'\geq w \ I\sep w' \not\modelsp \varphi
    \end{IEEEeqnarray*}
    In case that $w = t$, it follows that $w' = t$ and the result is trivial.
    Otherwise, $w=h$ and we have
    \begin{IEEEeqnarray*}{l ?C? l }
    I\sep w \modelsp \neg \varphi
      &\text{ iff }& \forall w'\geq w \ I\sep w' \not\modelsp \varphi
    \\
     &\text{ iff }& \text{both }  I\sep h \not\modelsp \varphi
     \text{ and } I\sep t \not\modelsp \varphi
    \end{IEEEeqnarray*}
    Furthermore, from Proposition~\ref{prop:preservation},
    it follows that
    $I\sep t \not\modelsp \varphi$
    implies
    $I\sep h \not\modelsp \varphi$.
    Hence, we get that
    $I\sep w \modelsp \neg \varphi$ holds
    iff
    $I\sep t \not\modelsp \varphi$.
    \\[10pt]
    For~\ref{item:5:prop:negation}, we have
    $I\sep w \modelsp \neg\neg\varphi$
    iff
    $I\sep w \modelsp \neg\varphi \to \bot$
    \\iff
    $\forall w'\geq w \ I\sep w' \not\modelsp \neg\varphi$
    \\iff
    $\forall w'\geq w \ I\sep t\modelsp \varphi$
    \\iff
    $I\sep t\modelsp \varphi$.
    \\[10pt]
    For~\ref{item:7:prop:negation},
    we have that
    $I\sep w \modelsp \neg\neg\neg\varphi$
    \\iff
    $I\sep w \modelsp \neg\neg\varphi \to \bot$
    \\iff
    $\forall w'\geq w \ I\sep w' \not\modelsp \neg\neg\varphi$
    \\iff
    $\forall w'\geq w \ I\sep t \not\modelsp \varphi$
    \\iff
    $I\sep t \not\modelsp \varphi$
    iff
    $I\sep w \modelsp \neg\varphi$.
    \\[10pt]
    For~\ref{item:2:prop:negation},
    we have
    $I\sep w \modelsn \neg\varphi$
    \\iff
    $I\sep w \modelsn \varphi \to \bot$
    \\iff
    $I\sep w \modelsp \varphi$ and $I\sep w \modelsn \bot$
    \\iff
    $I\sep w \modelsp \varphi$
    \\iff
    $I\sep w \modelsn \sneg\varphi$.
\end{proofof}

\begin{proofof}{Proposition~\ref{prop:imp.consitent}}
Note that
\mbox{$\sI\sep w \modelsp \varphi_1 \sup \varphi_2$}
holds iff either
\mbox{$\sI\sep w' \not\modelsp \varphi_1$}
or
\mbox{$\sI\sep w' \not\modelsp \neg\!\sneg\varphi_1$}
or
\mbox{$\sI\sep w' \modelsp \varphi_2$}
for all
\mbox{$w' \geq w$}.
In addition, it can be proved by induction that for every formula $\varphi_1$, consistent \Ninterpretation~$\sI$ and world~\mbox{$w' \in W$},
we have that
\mbox{$\sI\sep w' \modelsp \varphi_1$} implies \mbox{$\sI\sep w' \modelsp \neg\!\sneg\varphi_1$}.
As a result we can simplify the above equivalence as
$\sI\sep w \modelsp \varphi_1 \sup \varphi_2$ 
iff either
$\sI\sep w' \not\modelsp \varphi_1$
or $\sI\sep w' \modelsp \varphi_2$ for all $w' \geq w$.
By definition, this is equivalent to
$\sI\sep w \modelsp \varphi_1 \to \varphi_2$.
\end{proofof}

\subsubsection*{Auxiliary results for Proposition~\ref{prop:consitent}}

\begin{lemma}{\label{lem:default.negation.consitent}}
Given any \Ninterpretation~$\sI$ and any extended formula $\varphi$, 
then either
\mbox{$\sI\sep w \not\modelsp \Not\varphi$} or \mbox{$\sI\sep w \not\modelsn \Not\varphi$} holds.
\end{lemma}

\begin{proof}
Suppose, for the sake of contradiction, that
\mbox{$\sI\sep w \modelsp \Not\varphi$} and \mbox{$\sI\sep w \modelsn \Not\varphi$} hold.
Then, by definition, we have that the following two condition hold:
\begin{itemize}
\item $\sI\sep w \modelsp \neg \varphi \text{ or } \sI\sep w \modelsp \varphi \wedge \sneg\varphi$

\item $\sI\sep w \modelsp \varphi \text{ and } \sI\sep w \not\modelsn \varphi$
\end{itemize}
Then, latter implies that
$\sI\sep w \modelsp \neg \varphi$ and $\sI\sep w \modelsp \sneg\varphi$ do not hold,
which is a contradiction with the former.
Hence, either
\mbox{$\sI\sep w \not\modelsp \Not\varphi$} or \mbox{$\sI\sep w \not\modelsn \Not\varphi$} must hold.
\end{proof}

\begin{lemma}\label{lem:models.consitent}
Given any consistent \Ninterpretation~$\sI$ and any extended formula~$\varphi$, 
then we have:
\mbox{$\sI\sep w \models \varphi$} iff
\mbox{$\sI\sep w \modelsp \varphi$}.
\end{lemma}

\begin{proof}
By definition, we have:
\\
\piff
$\sI\sep w \models \varphi$
\\iff
$\sI\sep w \modelsp \neg\!\sneg \varphi \wedge \varphi$
\\iff
$\sI\sep w \modelsp \neg\!\sneg \varphi$ and  $\sI\sep w \modelsp \varphi$
\\iff
$\sI\sep w' \not\modelsp \sneg \varphi$ for all $w' \geq w$ and  $\sI\sep w \modelsp \varphi$ 
\\iff
$\sI\sep w' \not\modelsn \varphi$ for all $w' \geq w$ and  $\sI\sep w \modelsp \varphi$
\\
Finally, just note that since $\sI$ is consistent, it follows that
$\sI\sep w \modelsp \varphi$
implies
$\sI\sep w \not\modelsn \varphi$
and, in its turn, this implies
$\sI\sep w' \not\modelsn \varphi$ for all $w' \geq w$.
Therefore, we obtain that
$\sI\sep w \models \varphi$ iff \mbox{$\sI\sep w \modelsp \varphi$}
holds.
\end{proof}

\begin{lemma}{\label{lem:default.negation.aux1}}
Given any \Ninterpretation~$\sI$,
we have $\sI\sep w \modelsp \neg\!\sneg\Not \varphi$
iff
\begin{enumerate}[ label=\roman*) , leftmargin=10pt ]
\item[] 
$\sI\sep w' \not\modelsp \varphi$ or $\sI\sep w' \modelsn \varphi$ for all $w' \geq w$.
\label{item:1:lem:default.negation.aux1}
\end{enumerate}
\end{lemma}

\begin{proof}
We have:
\\
\piff
$\sI\sep w \modelsp \neg\!\sneg\Not \varphi$
\\iff
$\sI\sep w' \not\modelsp \sneg\Not \varphi$ for all $w' \geq w$
\\iff
$\sI\sep w' \not\modelsn \Not \varphi$ for all $w' \geq w$
\\iff
$\sI\sep w' \not\modelsp \varphi$ or $\sI\sep w' \modelsn \varphi$ for all $w' \geq w$
\end{proof}

\begin{lemma}\label{lem:negation.translation.aux}
Let $\sI$ be an \Ninterpretation, $\varphi$  be a formula
and \mbox{$w \in W$} be some world.
Then, we have: $\sI\sep w \modelsp \neg\!\sneg \Not \varphi$ iff
$\sI\sep w \modelsp \varphi \to \sneg\varphi$.
\end{lemma}

\begin{proof}
From Lemma~\ref{lem:default.negation.aux1},
we have:
\\
\piff
$\sI\sep w \modelsp \neg\!\sneg \Not \varphi$
\\
iff
$\sI\sep w' \not\modelsp \varphi \text{ or } \sI\sep w' \modelsn \varphi$ for all $w' \geq w$
\\
iff
$\sI\sep w' \not\modelsp \varphi \text{ or } \sI\sep w' \modelsp \sneg\varphi$ for all $w' \geq w$
\\
iff
$\sI\sep w \modelsp \varphi \to \sneg\varphi$.
\end{proof}

\begin{lemma}{\label{lem:default.negation.aux2}}
Any \Ninterpretation~$\sI$ and formula~$\varphi$ satisfy that
$\sI\sep w \modelsp \sneg\varphi$ implies $\sI\sep w \modelsp \neg\!\sneg\Not \varphi$ 
\end{lemma}

\begin{proof}
From Lemma~\ref{lem:default.negation.aux1},
we have:
\mbox{$\sI\sep w \modelsp \neg\!\sneg\Not \varphi$}
iff
\mbox{$\sI\sep w' \not\modelsp \varphi$}
or
\mbox{$\sI\sep w' \modelsn \varphi$} for all $w' \geq w$.
Then, just note that
$\sI\sep w \modelsp \sneg\varphi$ holds iff
$\sI\sep w \modelsn \varphi$ iff
$\sI\sep w' \not\modelsn \varphi$ for all $w' \geq w$ (by preservation).
\end{proof}

\begin{lemma}{\label{lem:default.negation.aux3}}
Any \Ninterpretation~$\sI$ and formula~$\varphi$ satisfy that
$\sI\sep w \modelsp \neg\varphi$ implies $\sI\sep w \modelsp \neg\!\sneg\Not \varphi$ 
\end{lemma}

\begin{proof}
From Lemma~\ref{lem:default.negation.aux1},
we have:
\mbox{$\sI\sep w \modelsp \neg\!\sneg\Not \varphi$}
iff
\mbox{$\sI\sep w' \not\modelsp \varphi$}
or
\mbox{$\sI\sep w' \modelsn \varphi$} for all $w' \geq w$.
Then, just note that
$\sI\sep w \modelsp \neg\varphi$ holds iff
$\sI\sep w' \not\modelsp \varphi$ for all $w' \geq w$.
\end{proof}

\begin{lemma}{\label{lem:default.negation}}
Given any \Ninterpretation~$\sI$ the following condition hold:
\begin{enumerate}[ label=\roman*) , leftmargin=20pt ]
\item $\sI\sep w \models \Not \varphi$ iff $\sI\sep w \modelsp \neg\varphi \vee (\varphi \wedge \sneg\varphi)$.
\label{item:1:lem:default.negation}
\end{enumerate}
\end{lemma}

\begin{proof}
By definition we have:
\\
\piff$\sI\sep w \models \Not \varphi$
\\iff
$\sI\sep w \modelsp \Not\varphi \wedge \neg\!\sneg\Not \varphi$
\\iff
$\sI\sep w \modelsp \Not\varphi$ and  $\sI\sep w \modelsp \neg\!\sneg\Not \varphi$
\\iff
$\sI\sep w \modelsp \neg\varphi \vee (\varphi \wedge \sneg\varphi)$ and  $\sI\sep w \modelsp \neg\!\sneg\Not \varphi$
\\iff both $\sI\sep w \modelsp \neg\varphi$ or  $\sI\sep w \modelsp \varphi \wedge \sneg\varphi$
\\
\piff and  $\sI\sep w \modelsp \neg\!\sneg\Not \varphi$
\\iff either $\sI\sep w \modelsp \neg\varphi$ and  $\sI\sep w \modelsp \neg\!\sneg\Not \varphi$ 
\\
\piff or  $\sI\sep w \modelsp \varphi \wedge \sneg\varphi$ and $\sI\sep w \modelsp \neg\!\sneg\Not \varphi$
\\iff $\sI\sep w \modelsp \neg\varphi$ or  $\sI\sep w \modelsp \varphi \wedge \sneg\varphi$ (Lemmas~\ref{lem:default.negation.aux2} and~\ref{lem:default.negation.aux3})
\\iff $\sI\sep w \modelsp \neg\varphi \vee (\varphi \wedge \sneg\varphi)$.
\end{proof}

\begin{proofof}{Proposition~\ref{prop:consitent}}
Condition~\ref{item:models:prop:consitent}
follows directly from Lemma~\ref{lem:models.consitent}.
This also implies that
\mbox{$\sI\sep w \models \neg\varphi$} holds
iff
\mbox{$\sI\sep w \modelsp \neg\varphi$}.
Furthermore, from Lemma~\ref{lem:default.negation},
it follows that
\mbox{$\sI \models \Not \varphi$} holds
iff
\mbox{$\sI\sep w \modelsp \neg\varphi \vee (\varphi \wedge \sneg\varphi)$}.
Hence, it is easy to check that
\mbox{$\sI\sep w \models \neg\varphi$}
implies
\mbox{$\sI \models \Not \varphi$}.
To show that the only if direction, just note that, since $\sI$ is consistent,
we have 
\mbox{$\sI\sep w \not\modelsp \varphi \wedge \sneg\varphi$}
and, therefore,
\mbox{$\sI \models \Not \varphi$}
also implies
\mbox{$\sI\sep w \models \neg\varphi$}.
That is, condition~\ref{item:neg:prop:consitent} holds.
\end{proofof}

\subsubsection*{Auxiliary results for Proposition~\ref{prop:cw-relation}}

\begin{lemma}\label{lem:sup.mponens}
Given any \Ninterpretation~$\sI$ and any pair of formulas $\varphi_1,\varphi_2$, the following condition holds:
\begin{enumerate}[ label=\roman*)]
\item $\sI \models \varphi_1$ and $\sI \modelsp \varphi_1 \sup \varphi_2$
imply $\sI \modelsp \varphi_2$
\label{item:1:prop:lem.models}
\end{enumerate}
\end{lemma}

\begin{proof}
By definition, we have
$\sI\sep w \modelsp \varphi_1 \sup \varphi_2$ holds iff
either 
$\sI\sep w' \not\modelsp \neg\!\sneg \varphi_1 \wedge \varphi$
or
$\sI\sep w' \modelsp \varphi_2$
for all $w' \geq w$.
Furthermore, by definition,
\mbox{$\sI\sep w \models \varphi_1$}
implies
\mbox{$\sI\sep w' \modelsp \neg\!\sneg \varphi \wedge \varphi$}
for all $w' \in W$
and, thus,~\ref{item:1:prop:lem.models} holds.
\end{proof}

\begin{proofof}{Proposition~\ref{prop:cw-relation}}
Condition~\ref{item:impl:prop:cw-relation} follows directly from Lemma~\ref{lem:sup.mponens}.
For condition~\ref{item:neg:1:prop:cw-relation}, note that
\mbox{$\sI \models \Not \varphi$}
implies that
\mbox{$\sI\sep w \modelsp \neg \varphi$}
or
\mbox{$\sI\sep w \modelsp \varphi \wedge \sneg \varphi$}
for all
\mbox{$w \in W$}.
Pick any world \mbox{$w \in W$}.
On the one hand,
\mbox{$\sI\sep w \modelsp \neg \varphi$} implies
\mbox{$\sI\sep w \not\modelsp \varphi$}
which, in its turn, implies
\mbox{$\sI\sep w \not\modelsp \varphi \wedge \neg\!\sneg \varphi$}
\mbox{$\sI\sep w \modelsp \varphi$}.
On the other hand
\mbox{$\sI\sep w \modelsp \varphi \wedge \sneg \varphi$}
implies
\mbox{$\sI\sep w \modelsp  \sneg \varphi$}
which implies
\mbox{$\sI\sep w \not\modelsp \varphi \wedge \neg\!\sneg \varphi$}
and
\mbox{$\sI\sep w \models \varphi$}.
Hence, implies that
\mbox{$\sI\sep w \not\models \varphi$}
for all
\mbox{$w \in W$}
and 
\mbox{$\sI \not\models \varphi$}.
Furthermore, if $\sI$ is a total \HTinterpretation, we have that
\mbox{$\sI \not\models \varphi$}
implies
\mbox{$\sI\sep t \not\modelsp \varphi \wedge \neg\!\sneg\varphi$}
which implies
\mbox{$\sI\sep t \not\modelsp \varphi$}
and
\mbox{$\sI\sep t \modelsp \neg\varphi$}.
Hence,
\mbox{$\sI \models \Not \varphi$}
and condition~\ref{item:neg:prop:cw-relation} hold.
Note that this does not hold if $\sI$ is not total: take \mbox{$\sI = \tuple{\bH,\bT}$}
with \mbox{$\bH = \set{\sneg a}$} and \mbox{$\bT = \set{ a , \sneg a}$}.
Then, $\sI \modelsp \sneg a$ and $\sI \not\modelsp \neg a$.
The latter implies that
\mbox{$\sI \not\models \varphi$}
while we can check that
\mbox{$\sI \not\models \Not\varphi$}
because $\sI \not\modelsp \neg a$
and $\sI \not\modelsp a$.
\end{proofof}

\begin{proofof}{Proposition~\ref{prop:default.negation.alt}}
By definition,
$\sI \models \neg \varphi \vee (\varphi \wedge \sneg\varphi)$
\\iff $\sI \modelsp (\neg \varphi \vee (\varphi \wedge \sneg\varphi)) 
        \wedge \neg\!\sneg\,(\neg \varphi \vee (\varphi \wedge \sneg\varphi))$
\\iff $\sI \modelsp (\neg \varphi \vee (\varphi \wedge \sneg\varphi)) 
        \wedge (\neg\!\sneg\neg \varphi \wedge (\neg\!\sneg\varphi \vee \neg\!\sneg\sneg\varphi))$
\\iff $\sI \modelsp (\neg \varphi \vee (\varphi \wedge \sneg\varphi)) 
        \wedge (\neg\!\sneg\sneg \varphi \wedge (\neg\!\sneg\varphi \vee \neg\!\sneg\sneg\varphi))$ \hfill (Proposition~\ref{prop:negation} iv)
\\iff $\sI \modelsp (\neg \varphi \vee (\varphi \wedge \sneg\varphi)) 
        \wedge (\neg \varphi \wedge (\neg\!\sneg\varphi \vee \neg\varphi))$
\\iff $\sI \modelsp (\neg \varphi \vee (\varphi \wedge \sneg\varphi)) 
        \wedge \neg \varphi$
\\iff $\sI \modelsp \neg \varphi$
\\iff $\sI \modelsp \neg \varphi \wedge \neg\varphi$
\\iff $\sI \modelsp \neg \varphi \wedge \neg\!\sneg\sneg\varphi$
\\iff $\sI \modelsp \neg \varphi \wedge \neg\!\sneg\neg\varphi$  \hfill (Proposition~\ref{prop:negation} iv)
\\iff $\sI \models \neg \varphi$
\end{proofof}

\begin{proofof}{Proposition~\ref{prop:att.mponens}}
By definition,
we get that
${\sI\sep w \modelsp \varphi_1 \att \varphi_2}$
holds
iff
${\sI\sep w \modelsp \varphi_1 \sup \sneg\varphi_2}$.
Furthermore, from Proposition~\ref{prop:cw-relation},
we get that
$\sI\sep w \modelsp \varphi_1$
implies
$\sI\sep w \models \sneg\varphi_2$.
As a result, we get that $\sI\sep w \modelsn \varphi_2$.
\end{proofof}

\subsubsection*{Auxiliary results for Proposition~\ref{prop:af.model} and Theorem~\ref{thm:sf.stable<->emodel}}

\begin{proposition}\label{prop:attack.models}
Given any \Ninterpretation~$\sI$ and any pair of formulas $\varphi_1,\varphi_2$, the following conditions are equivalent:
\begin{enumerate}[ label=\roman*), leftmargin=17pt]
\item $\sI\sep w \modelsp \varphi_1 \att \varphi_2$,
\label{item:1:prop:attack.models}
\item $\sI\sep w \modelsp \neg\!\sneg\varphi_1 \wedge \varphi_1 \to \sneg \varphi_2$,
\label{item:2:prop:attack.models}
\end{enumerate}
Furthermore, if $\sI$ is a \HTinterpretation, then
\begin{enumerate}[ label=\roman*), start=3, leftmargin=17pt]
\item $\sI\sep w' \not\modelsp \varphi_1$
or \hspace{1pt} $\sI\sep t \modelsn \varphi_1$
or \hspace{1pt} $\sI\sep w' \modelsn \varphi_2$ 
\hspace{1pt} for all $w' \geq w$.
\label{item:3:prop:attack.models}
\end{enumerate}
\end{proposition}

\begin{proof}
By definition, we have that\\
$\sI\sep w \modelsp \varphi_1 \att \varphi_2$
\\iff
$\sI\sep w \modelsp \varphi_1 \sup \sneg \varphi_2$
\\iff
$\sI\sep w \modelsp \neg\!\sneg \varphi_1 \wedge \varphi_1 \to \sneg \varphi_2$
\\
Hence, conditions~\ref{item:1:prop:attack.models} and~\ref{item:2:prop:attack.models} are equivalent. Furthermore, we also have that
\\
$\sI\sep w \modelsp \neg\!\sneg\varphi_1 \wedge \varphi_1 \to \sneg \varphi_2$
\\iff
$\sI\sep w' \not\modelsp \neg\!\sneg\varphi_1 \wedge \varphi_1$ or  $\sI\sep w' \modelsp  \varphi_2$ for all $w' \geq w$
\\iff
$\sI\sep w' \not\modelsp \neg\!\sneg\varphi_1$ or  $\sI\sep w' \not\modelsp \varphi_1$ or  $\sI\sep w' \modelsp  \varphi_2$ for all $w' \geq w$
\\iff
$\sI\sep t \modelsp \sneg\varphi_1$ or  $\sI\sep w' \not\modelsp \varphi_1$ or  $\sI\sep w' \modelsp  \varphi_2$ for all $w' \geq w$ \hfill(Proposition~\ref{prop:negation})
\\iff
$\sI\sep t \modelsn \varphi_1$ or  $\sI\sep w' \not\modelsp \varphi_1$ or  $\sI\sep w' \modelsp  \varphi_2$ for all $w' \geq w$.
\\
Hence, the three conditions are equivalent.
\end{proof}

\begin{lemma}\label{lem:model->defeated}
Let \mbox{$\EBAF$} for some set attack framework and let $\sI$ be some model of $\LEF{\EBAF}$.
Then, we have that $\Defeated{\SI} \subseteq H_\sI^-$.
\end{lemma}

\begin{proof}
\def\aA{\SI}
Pick any $a \in \Defeated{\aA}$.
By definition, there is $B \subseteq \aA$ such that \mbox{$(B,a) \in \R_a$}
and, thus, $\bigwedge B \att a$ belongs to $\LEF{\EBAF}$.
Furthermore, $B \subseteq \aA$ implies
$\sI \models \bigwedge B$
and, thus, \mbox{$\sI \modelsn a$} (Proposition~\ref{prop:att.mponens}).
In its turn, this implies that $a \in H_\sI^-$.
\end{proof}

\begin{lemma}\label{lem:model->cfree}
Let $\sI$ be a model of $\LEF{\EBAF}$ for some framework \mbox{$\EBAF$}.
Then,
$\SI$ is conflict-free.
\end{lemma}

\begin{proof}
From Lemma~\ref{lem:model->defeated}
it follows that $\Defeated{\SI} \subseteq H_\sI^- \subseteq T_\sI^-$.
Furthermore, by definition, 
$a \in \SI$ iff $\sI \models a$
$\sI \models a \wedge \neg\!\sneg a$
iff $\sI,h \modelsp a$ and $\sI,t \not\modelsn a$.
Hence, $\SI \cap T_\sI^- = \varnothing$
and, thus, we have that $\SI \cap \Defeated{\SI} = \varnothing$.
In other words, the set $\SI$ is \cfree.
\end{proof}

\begin{lemma}\label{lem:sf.emodel->defeated}
Let \mbox{$\SF$} for some framework and $\sI$ be some $\leq$-minimal model of~$\LSF{\SF}$.
Then, we have that $\Defeated{\SI} = H_\sI^-$.
\end{lemma}

\begin{proof}
\def\aA{\SI}
From Lemma~\ref{lem:model->defeated},
it follows that 
\mbox{$\Defeated{\SI} \subseteq H^-_\sI$}.
To show that
\mbox{$\Defeated{\SI} = H^-_\sI$}
also holds, pick any
\mbox{$a \in H^-_\sI$}, that is, we have that
\mbox{$\sI \modelsn a$}.
Let $\sJ$ be an \HTinterpretation\ with
\mbox{$\bT_\sJ = \bT_\sI$},
\mbox{$H_\sJ^+ = H_\sI^+$} and
\mbox{$H_\sJ^- = H_\sI^- \setminus \set{ a }$}.
Note that, by construction, we have \mbox{$\sJ < \sI$}.
Then, since $\sI$ is a $\leq$-minimal model,
it must be that $\sJ$ is not a model of $\LSF{\SF}$.
Furthermore, since $H_\sJ^+ = H_\sI^+$ and $\sI \models \A$,
we have that $\sJ \models \A$ as well.
Thus, there must be some attack
$(B,a) \in \R_a$ such that
\mbox{$\sJ \not\models^+ (\bigwedge B \att a)$}.
This implies $\sJ \models \bigwedge B$
which, in its turn, implies
$B \subseteq \SI$.
Hence, we have
\mbox{$a \in \Defeated{\aA}$}
and, thus,
\mbox{$\Defeated{\SI} = H^-_\sI$}.
\end{proof}

\begin{proposition}\label{prop:sf.model}
Let \mbox{$\SF$} for some framework and
\mbox{$\sI$} be some model of $\LSF{\SF}$.
Then, the following statement hold:
\begin{enumerate}[ label=\roman*), leftmargin=20pt]

\item if $a$ is defeated \wrt~$\SI$, then $\sI \modelsp \sneg a$

\item $\SI$ is \cfree.
\end{enumerate}
If, in addition, $\sI$ is an equilibrium model, then
\begin{enumerate}[ label=\roman*), start=3, leftmargin=20pt]

\item $a$ is defeated \wrt~$\SI$ iff $\sI \modelsp \sneg a$.
\end{enumerate}
\end{proposition}

\begin{proof}
Directly follows from Lemmas~\ref{lem:model->defeated},~\ref{lem:model->cfree}
and~\ref{lem:sf.emodel->defeated}, respectively.
\end{proof}

\begin{proofof}{Proposition~\ref{prop:af.model}}
It is a direct consequence of Proposition~\ref{prop:sf.model} because AFs are a particular case of SETAFs.
\end{proofof}

Proof of Theorem~\ref{thm:af.stable<->emodel} follows as particular case of Theorem~\ref{thm:sf.stable<->emodel}.

\begin{proofof}{Theorem~\ref{thm:sf.stable<->emodel}}
To prove note that~\ref{item:1:thm:stable<->emodel},
it is enough to show $\SI = \overline{\Defeated{\SI}}$.
Note that, since $\sI$ is a model of $\LSF{\SF}$,
we have that $\A \subseteq T_\sI^+$ and, thus,
$\SI = \A \setminus T_\sI^-$.
Furthermore, from Lemma~\ref{lem:sf.emodel->defeated},
this implies
$\SI = \A \setminus \Defeated{\SI} = \overline{\Defeated{\SI}}$.
\\[10pt]
Let us now show~\ref{item:2:thm:stable<->emodel}.
Since $S$ is a stable extension,
we have
$S  = \overline{\Defeated{S}} = \A \setminus T_\sI^- = \SI$.
Hence, $\SI$ is a stable extension and, from Lemma~\ref{lem:model.theory}, it follows that $\sI$ is a model of $\LSF{\SF}$.
Suppose, for the sake of contradiction, that $\sI$ is not an equilibrium model.
Then, there is an \HTinterpretation\ $\sJ < \sI$ that is a model of $\LSF{\SF}$
and, from Lemma~\ref{lem:model->defeated},
it follows that
\begin{IEEEeqnarray*}{l ,C, l ,C, l ,C, l ,C, l}
\Defeated{\SJ} &\subseteq& H_\sJ^- &\subseteq& H_\sI^- &=& \Defeated{\SI}
\end{IEEEeqnarray*}
Furthermore, we have that $\A \subseteq H_\sJ^+ \subseteq T_\sI^+ = \A$
and, thus, $\sJ < \sI$
implies
\mbox{$H_\sJ^- \subset H_\sI^- = \Defeated{\SI}$}.
Pick any argument
\mbox{$a \in \Defeated{\SI} \setminus H_\sJ^-$}.
Then, there is some $(B,a) \in \R_a$ such that $B \subseteq \SI$
and, thus, we have that
$\sI \models \bigwedge B$.
That is
$\sI\sep h \modelsp \bigwedge B$
and
$\sI\sep t \modelsn \bigwedge B$.
Note that, since $H_\sI^+ = H_\sJ^+$ and $T_\sI^- = T_\sJ^-$,
this immediately implies 
$\sJ\sep h \modelsp \bigwedge B$
and
$\sJ\sep t \modelsn \bigwedge B$
and, thus, that
$\sJ \models \bigwedge B$ holds.
Furthermore,
$(B,a) \in \R_a$
implies that
$\bigwedge B \att a$ belongs to $\LSF{\SF}$
and, thus, that
$\sJ\sep h \modelsn a$
which implies
$a \in H_\sJ^-$.
This is a contradiction with the fact that
\mbox{$a \in \Defeated{\SI} \setminus H_\sJ^-$}.
Consequently, $\sI$ is an equilibrium model.
\\[10pt]
Let show now that this determines a one-to-one correspondence.
Let $\sI$ and $\sJ$ be two equilibrium models such that $\SI = \SJ$.
Then, from Lemma~\ref{lem:sf.emodel->defeated} 
it follows that 
$\Defeated{\SI} = T_\sI^+$ and
$\Defeated{\SJ} = T_\sJ^+$
and, thus, we have $T_\sI^- = T_\sJ^-$.
Hence, $\sI = \sJ$.
The other way around. Let $S_1$ and $S_2$ be two stable extensions such that
$\Defeated{S_1} = \Defeated{S_2}$.
Note that, since $S_1$ and $S_2$ are stable extensions,
we have that $S_i = \overline{\Defeated{S_i}}$ with $i \in \set{1,2}$ and, thus, $S_1 = S_2$.
\end{proofof}

\subsubsection*{Auxiliary results for Proposition~\ref{prop:ef.model}}

\begin{lemma}
\label{lem:model->supported}
Let \mbox{$\EBAF$} for some framework and
\mbox{$\sI$} be some model of $\LEF{\EBAF}$.
Then, $\Supported{\SI} \subseteq H_\sI^+$.
\end{lemma}

\begin{proof}
We will prove the following stronger result:
\begin{itemize}
\item[] $\Supported{S} \subseteq H_\sI^+$ for every set $S \subseteq \SI$.
\end{itemize}
First, note that if $S = \varnothing$, then 
$\Supported{S} = \PF$.
Besides, by definition, \mbox{$\LEF{\EBAF} \supseteq \PF$}
and, thus, we have that $\PF \subseteq H_\sI^+$.
Otherwise, we proceed by induction assuming the the above statement holds for all strict subsets of $S$.
Pick any supported argument
\mbox{$a \in \Supported{S}$}.
By definition, there is some $B \subseteq S \cap \Supported{ S \setminus \set{ a} }$ such that  $(B,a) \in \R_s$.
Hence, every $b \in B$ satisfies $b \in \Supported{ S \setminus \set{ a, b} }$ \cite[Lemma A.11]{Ca2018.5} and $b \in S$.
These two facts together imply
\mbox{$S \setminus \set{ a, b} \subset S$} and, by induction hypothesis, it follows then that
\mbox{$b \in \Supported{ S \setminus \set{ a, b} } \subseteq H_\sI^+$}.
Hence,
we have
\mbox{$B \subseteq H_\sI^+$}.
Furthermore, $(B,a) \in \R_s$ implies that $\bigwedge B \sup a$ belongs to~$\LEF{\EBAF}$.
In addition,
$B \subseteq S \subseteq \SI$
implies that $\sI \models \bigwedge B$.
Since 
\mbox{$\bigwedge B \sup a$} belongs to~$\LEF{\EBAF}$,
this implies that
\mbox{$a \in H_\sI^+$}.
Hence, we have that $\Supported{S} \subseteq H_\sI^+$ for every $S \subseteq \SI$
and, in particular, for $S = \SI$.
\end{proof}

\begin{lemma}\label{lem:min.model->supported}
Let \mbox{$\EBAF$} for some framework and $\sI$ be some $\leq$-minimal model of~$\LEF{\EBAF}$.
Then, we have that $\Supported{\SI} = H_\sI^+$.
\end{lemma}

\begin{proof}
First note that,
from Lemma~\ref{lem:model->supported},
we have that
\mbox{$\Supported{\SI} \subseteq H_\sI^+$}.
We will show now that every $\sJ$ with
\mbox{$\PF \subseteq H_\sJ^+ \subseteq H_\sI^+$},
\mbox{$H_\sJ^- = H_\sI^-$},
\mbox{$T_\sJ^+ = T_\sI^+$} and
\mbox{$T_\sJ^+ = T_\sI^-$}
satisfies
\mbox{$\Supported{\SJ} \supseteq H_\sJ^+$}.
Note that $\PF \subseteq H_\sI^+$ follows from the fact that $\sI$ is a model of
\mbox{$\LEF{\EBAF} \supseteq \PF$}.
Assume as induction hypothesis that the statement holds for all \mbox{$\sK < \sJ$}.
Clearly, \mbox{$\sK < \sI$} and, thus, \mbox{$\sK \not\models \LEF{\EBAF}$} follows.
Then, there is a formula of the form \mbox{$\bigwedge B \sup a$} in $\LEF{\EBAF}$ which is not satisfied by~$\sK$.
This implies that
\mbox{$\sK \models \bigwedge B$}
and, thus, that
\mbox{$\sJ \models \bigwedge B$}.
These two facts respectively imply
\mbox{$B \subseteq \SK \subseteq H_{\sK}^+$}
and
\mbox{$B \subseteq \SJ \setminus \set{a}$}.
Furthermore, by induction hypothesis, we have
\mbox{$\Supported{\SK} \supseteq H_{\sK}^+ \supseteq B$}.
Note that, by construction,
\mbox{$\SK \subseteq \SJ \setminus \set{ a}$} holds because \mbox{$\sK \not\models a$}.
Thus,
\mbox{$B \subseteq \Supported{\SJ \setminus \set{ a}} \cap (\SJ \setminus \set{a})$}.
Finally, 
since \mbox{$\bigwedge B \sup a$} belongs to $\LEF{\EBAF}$,
we have $(B,a) \in \R_s$
and, thus, that
\mbox{$a \in \Supported{\SJ}$}.
Consequently, \mbox{$\Supported{\SJ} \supseteq H_\sJ^+$}
holds for all $\sJ \leq \sI$ and, in particular, for $\sI$.
\end{proof}

\begin{lemma}\label{lem:model.theory.supported}
Let $\EBAF$ be some framework
and $\sI$ be an \HTinterpretation\ with \mbox{$H_\sI^+ = \Supported{\SI}$}.
Then, we have that $\sI$ is a model of $(\Gamma_{\R_s} \cup \PF)$.
\end{lemma}

\begin{proof}
Suppose, for the sake of contradiction, that $\sI$ is not a model of
\mbox{$(\Gamma_{\R_s} \cup \PF)$}.
Then, 
either there is 
\mbox{$a \in \PF$}
such that
\mbox{$\sI \not\modelsp a$}
or there is
\mbox{$(B,a) \in \R_s$}
such that \mbox{$(\bigwedge B \sup a) \in \Gamma_{\R_s}$}
is not satisfied by $\sI$.
The former implies that \mbox{$a \in \Supported{\SI} = H_\sI^+$}
which is a contradiction with
\mbox{$\sI \not\modelsp a$},
so we may assume without loss of generality the latter.
This implies that
\mbox{$\sI \models \bigwedge B$}
and
\mbox{$\sI \not\modelsp a$}.
Note that
\mbox{$\sI \not\modelsp a$}
implies that $a \notin \SI$
and, thus, $\SI \setminus\set{a} = \SI$.
Then,
\mbox{$\sI \models \bigwedge B$} implies
\mbox{$B \subseteq \SI \setminus \set{a}\subseteq H_\sI^+ = \Supported{\SI} = \Supported{\SI\setminus\set{a}}$}
This
implies that
\mbox{$B \subseteq (\SI \setminus \set{a}) \cap \Supported{\SI\setminus\set{a}}$}
and, since $(B,a) \in \R_s$, that $a \in \Supported{\SI} = H_\sI^+$,
which is a contradiction
with \mbox{$\sI \not\modelsp a$}.
Consequently, 
$\sI$ is a model of
\mbox{$(\Gamma_{\R_s} \cup \PF)$}.
\end{proof}

\begin{lemma}\label{lem:min.model->defeated}
Let \mbox{$\EBAF$} for some framework and $\sI$ be some $\leq$-minimal model of~$\LSF{\EBAF}$.
Then, we have that $\Defeated{\SI} = H_\sI^-$.
\end{lemma}

\begin{proof}
\def\aA{\SI}
First note that, from Lemma~\ref{lem:model->defeated},
it follows that 
\mbox{$\Defeated{\SI} \subseteq H^-_\sI$}.
To show that
\mbox{$\Defeated{\SI} = H^-_\sI$}
also holds, pick any
\mbox{$a \in H^-_\sI$}.
Let $\sJ$ be an \HTinterpretation\ with
\mbox{$\bT_\sJ = \bT_\sI$},
\mbox{$H_\sJ^+ = H_\sI^+$} and
\mbox{$H_\sJ^- = H_\sI^- \setminus \set{ a }$}.
Note that, by construction, we have \mbox{$\sJ < \sI$}.
Then, since $\sI$ is a $\leq$-minimal model,
it must be that $\sJ$ is not a model of $\LSF{\SF}$.
Furthermore, from Lemma~\ref{lem:min.model->supported},
it follows that
\mbox{$\Supported{\SI} = H_\sI^+ =H_\sJ^+$}.
Furthermore, by construction, we have that $\SI = \SJ$
and, thus, we have
\mbox{$\Supported{\SJ} = H_\sJ^+$}.
Hence,
\mbox{$\sJ \models \PF \cup \Gamma_{\R_s}$}
follows directly from Lemma~\ref{lem:model.theory.supported}.
Therefore, there must be some attack
$(B,a) \in \R_a$ such that
\mbox{$\sJ \not\models^+ (\bigwedge B \att a)$}.
This implies $\sJ \models \bigwedge B$
which, in its turn, implies that
\mbox{$B \subseteq \SJ = \SI$}.
Hence, we have that
\mbox{$a \in \Defeated{\aA}$}
and, thus,
\mbox{$\Defeated{\SI} = H^-_\sI$}.
\end{proof}

\begin{proofof}{Proposition~\ref{prop:ef.model}}
Conditions~\ref{item:1:prop:ef.model},~\ref{item:2:prop:ef.model} and~\ref{item:3:prop:ef.model} follow directly from Lemmas~\ref{lem:model->supported}, \ref{lem:model->defeated} and~\ref{lem:model->cfree}.
Similarly, conditions~\ref{item:3:prop:ef.model} and~\ref{item:4:prop:ef.model}
follow directly from Lemmas~\ref{lem:min.model->supported} and~\ref{lem:min.model->defeated}.
Finally,~\ref{item:6:prop:ef.model} follows from Lemma~\ref{lem:min.model->defeated} the fact that, by construction, we have $\SI \subseteq H_\sI^+$.
\end{proofof}

\subsubsection*{Auxiliary results for Proposition~\ref{prop:ef.model.supportable}}

\begin{definition}
Let $\CanSupported{\aA}$ be the set of all supportable arguments \wrt\ some set~$\aA$.
\end{definition}

\begin{lemma}\label{lem:model->cansupported}
    Let \mbox{$\EBAF$} for some framework and
    \mbox{$\sI$} be some model of $\LEF{\EBAF}$
    such that \mbox{$\Defeated{\SI} \supseteq T^-_\sI$}.
    Then, we have that $\CanSupported{\SI} \subseteq H_\sI^+$.
\end{lemma}

\begin{proof}
    The proof is similar to that of Lemma~\ref{lem:model->supported}.
    We will prove the following stronger result:
    \begin{itemize}
      \item[] $\Supported{S} \subseteq H_\sI^+$ for every set $S \subseteq \overline{\Defeated{\SI}}$.
    \end{itemize}
    First, note that if $S = \varnothing$, then
    $\CanSupported{S} = \PF$.
    Besides, by definition, \mbox{$\LEF{\EBAF} \supseteq \PF$}
    and, thus, we have that $\PF \subseteq H_\sI^+$.
    Otherwise, we proceed by induction assuming the the above statement holds for all strict subsets of $S$.
    Pick any supported argument
    \mbox{$a \in \Supported{S}$}.
    By definition, there is some $B \subseteq S \cap \Supported{ S \setminus \set{ a} }$ such that  $(B,a) \in \R_s$.
    Hence, every $b \in B$ satisfies $b \in \Supported{ S \setminus \set{ a, b} }$ \cite[Lemma A.11]{Ca2018.5} and $b \in S$.
    These two facts together imply
    \mbox{$S \setminus \set{ a, b} \subset S$} and, by induction hypothesis, it follows then that
    \mbox{$b \in \Supported{ S \setminus \set{ a, b} } \subseteq H_\sI^+$}.
    Hence,
    we have
    \mbox{$B \subseteq H_\sI^+$}.
    Furthermore, $(B,a) \in \R_s$ implies that $\bigwedge B \sup a$ belongs to~$\LEF{\EBAF}$.
    In addition,
    $B \subseteq S \subseteq \overline{\Defeated{\SI}} \subseteq \A \setminus T_\sI^-$
    implies that $\sI \models \bigwedge B$.
    Since
    \mbox{$\bigwedge B \sup a$} belongs to~$\LEF{\EBAF}$,
    this implies that
    \mbox{$a \in H_\sI^+$}.
    Hence, we have that $\Supported{S} \subseteq H_\sI^+$ for every $S \subseteq \SI$
    and, in particular, for $S = \SI$.
\end{proof}

\begin{lemma}\label{lem:min.model->cansupported}
    Let \mbox{$\EBAF$} for some framework and $\sI$ be some $\leq$-minimal model of~$\LEF{\EBAF}$
    such that $\Defeated{\SI} \supseteq T^-_\sI$.
    Then, $\CanSupported{\SI} = H_\sI^+$.
\end{lemma}

\begin{proof}
    First note that,
    from Lemma~\ref{lem:model->cansupported},
    we have that
    \mbox{$\CanSupported{\SI} \subseteq H_\sI^+$}.
    Furthermore, from Lemma~\ref{lem:model->cfree},
    we have that $\SI$ is \cfree\ and, thus, $\SI \subseteq \overline{\Defeated{\SI}}$.
    This implies that
    \begin{gather*}
    \Supported{\SI} 
        \ \subseteq \ \Supported{\overline{\Defeated{\SI}}}
        \ = \ \CanSupported{\SI}
    \end{gather*}
    Finally, from Lemma~\ref{lem:min.model->supported},
    we have that $\Supported{\SI} = H_\sI^+$
    and, thus, 
    $H_\sI^+ \subseteq \CanSupported{\SI}$
    follows.
\end{proof}

\begin{proofof}{Proposition~\ref{prop:ef.model.supportable}}
    First note that, since $\sI$ is an equilibrium model,
    it is also a minimal model.
    Then, from Lemma~\ref{lem:min.model->defeated},
    we have $\Defeated{\SI} = H_\sI^-$.
Furthermore,
    since $\sI$ is an equilibrium model,
    it is a $\leq$-minimal model.
Then, from Lemma~\ref{lem:min.model->cansupported},
    it follows
    $\CanSupported{\SI} = H_\sI^+$
Finally,  since $\sI$ is an equilibrium model,
    it is also a total \HTinterpretation\ and, thus,
    we have that $H_\sI^+ = T_\sI^+$ and $H_\sI^- = T_\sI^-$.
This implies that
    $\CanSupported{\SI} = T_\sI^+$.
\end{proofof}

\subsubsection*{Auxiliary results for Theorem~\ref{thm:ef.stable<->emodel}}

\begin{lemma}
\label{lem:model.theory.defeated}
Let $\EBAF$ be some framework and $\sI$ be an \HTinterpretation\ with \mbox{$H_\sI^- \supseteq \Defeated{\SI}$}.
Then, $\sI$ is a model of $\Gamma_{\R_a}$.
\end{lemma}

\begin{proof}
Suppose, for the sake of contradiction, that $\sI$ is not a model of $\Gamma_{\R_a}$.
Then, there is 
$(B,a) \in \R_a$
such that \mbox{$(\bigwedge B \att a) \in \Gamma_{\R_a}$}
is not satisfied by $\sI$.
This implies that
\mbox{$\sI \models \bigwedge B$}
and
\mbox{$\sI \not\modelsn a$}.
The former implies that
\mbox{$B \subseteq \SI$}
while the latter implies
\mbox{$a \notin H_\sI^-$}.
Furthermore, by hypothesis,
\mbox{$a \notin H_\sI^-$}
implies
\mbox{$a \notin \Defeated{\SI}$}
which is a contradiction
with the fact that
\mbox{$B \subseteq \SI$}
and
\mbox{$(B,a) \in \R_a$}
hold.
\end{proof}

\begin{lemma}\label{lem:model.theory}
Let $\EBAF$ be some framework.
Then, every \HTinterpretation~$\sI$ satisfying \mbox{$H_\sI^+ = \Supported{\SI}$} and \mbox{$H_\sI^- \supseteq \Defeated{\SI}$}
is a model of~$\LEF{\EBAF}$.
\end{lemma}

\begin{proof}
Directly from Lemmas~\ref{lem:model.theory.supported} and~\ref{lem:model.theory.defeated} taken together.
\end{proof}

\begin{lemma}\label{lem:emodel->stable}
Let \mbox{$\EBAF$} be some framework and $\sI$ be some equilibrium model of~$\LEF{\EBAF}$.
Then, $\SI$ is a stable extension of $\EBAF$.
\end{lemma}

\begin{proof}
First note that, since $\sI$ is an equilibrium model, it is also a total model
and, thus, $H_\sI^+  = T_\sI^+$
and we have that
$\SI = H_\sI^+ \setminus T_\sI^- = T_\sI^+ \setminus T_\sI^-$.
Furthermore, since $\sI$ is an equilibrium model, it is also $\leq$-minimal and,
from Lemmas~\ref{lem:min.model->defeated} and Proposition~\ref{prop:ef.model.supportable},
this implies that
\mbox{$T_\sI^- = \Defeated{\SI}$}
and
\mbox{$T_\sI^+ = \CanSupported{\SI}$}.
Hence, we get
\begin{align*}
\SI \ &= \ \CanSupported{\SI} \cap \overline{\Defeated{\SI}} 
\\  \ &= \ \overline{\overline{\CanSupported{\SI}} \cup \Defeated{\SI}} 
\\  \ &= \ \overline{\UnAcceptable{\SI}}
\end{align*}
By definition, this implies that $\SI$ is a stable extension of $\EBAF$.
\end{proof}

\begin{lemma}\label{lem:int.leq}
Let $\sI$ and $\sJ$ be a pair of interpretations such that $\sJ \leq \sI$.
Then, we have that $\SJ \subseteq \SI$.
\end{lemma}

\begin{proof}
Pick any $a \in \SJ$.
Then, $\sJ \models a \wedge \neg\sneg a$
and, thus,
we have that
$\sJ,w \modelsp a$ and  $\sJ,t \not\modelsn a$.
This implies that $a \in H_\sJ^+ \subseteq H_\sI^+$ 
and $a \notin T_\sJ^- = T_\sI^-$.
Consequently, $a \in \SI$.
\end{proof}

\begin{lemma}
\label{lem:stable->model}
Let \mbox{$\EBAF$} for some framework and $\sI$ be a total interpretation
such that $\SI$ is a stable extension of $\EBAF$, $T_\sI^+ = \Supported{\SI}$
and $T_\sI^- = \Defeated{\SI}$.
Then, $\sI$ is an equilibrium model of $\LEF{\EBAF}$.
\end{lemma}

\begin{proof}
Since $\sI$ is a total \HTinterpretation, we have
\mbox{$H_\sI^+ = T_\sI^+ = \Supported{\SI}$}
and
\mbox{$H_\sI^- = T_\sI^- = \Defeated{\SI}$}.
From Lemma~\ref{lem:model.theory},
this implies that
$\sI$ is an model of $\LEF{\EBAF}$.
Suppose, for the sake of contradiction, that it is not an equilibrium model.
Then, there is an \HTinterpretation\ \mbox{$\sJ < \sI$} that is a model of $\LEF{\EBAF}$.
This plus the lemma hypothesis \mbox{$T_\sI^- = \Defeated{\SI}$}
imply
$T_\sJ^- = T_\sI^- = \Defeated{\SI}$.
Then, from Lemma~\ref{lem:model->cansupported},
it follows
\begin{gather}
\CanSupported{\SJ} \subseteq H_\sJ^+ \subseteq
T_\sI^+ = \Supported{\SI} \subseteq \CanSupported{\SI}
 \label{eq:1:lem:stable->model}
\end{gather}
Furthermore,
from Lemma~\ref{lem:int.leq},
the fact \mbox{$\sJ < \sI$} also implies
\mbox{$\SJ \subseteq \SI$} which,
from Lemma A.3 in~\cite{Ca2018.5}, implies
$\CanSupported{\SJ} \supseteq \CanSupported{\SI}$.
This plus~\eqref{eq:1:lem:stable->model}
imply
\mbox{$H_\sJ^+ = H_\sI^+ = T_\sI^+$}.
In its turn, this plus
\mbox{$\sJ < \sI$}
imply
\begin{gather*}
H_\sJ^- \subset H_\sI^- = T_\sI^- = \Defeated{\SI}
\end{gather*}
Pick any
\mbox{$a \in \Defeated{\SI} \setminus H_\sJ^-$}.
Then, there is some $(B,a) \in \R_a$ such that $B \subseteq \SI$
and, thus, we have that
$\sI\sep h \modelsp \bigwedge B$
and
$\sI\sep t \modelsn \bigwedge B$.
Note that, since $H_\sI^+ = H_\sJ^+$ and $T_\sI^- = T_\sJ^-$,
this immediately implies that
\mbox{$\sJ\sep h \modelsp \bigwedge B$}
and
\mbox{$\sJ\sep t \modelsn \bigwedge B$}
also hold.
Furthermore,
\mbox{$(B,a) \in \R_a$}
implies that
$\bigwedge B \att a$ belongs to $\LEF{\EBAF}$
and, thus, that
$\sJ\sep h \modelsn a$
which implies
\mbox{$a \in H_\sJ^-$}.
This is a contradiction with the fact that
\mbox{$a \in \Defeated{\SI} \setminus H_\sJ^-$}.
Consequently, $\sI$ is an equilibrium model.
\end{proof}

\begin{proofof}{Theorem~\ref{thm:ef.stable<->emodel}}
First note that~\ref{item:1:thm:stable<->emodel} follows directly from Lemma~\ref{lem:emodel->stable}.
Furthermore, since $S$ is a stable extension and stable extensions are \ssupporting\ sets~\cite[Theorem~2]{Fandinno2018foiks}
we have
\begin{IEEEeqnarray*}{ l ,C, rCcCl ,C, l ,C, l}
S  &=& \CanSupported{S} &\cap& \overline{\Defeated{S}} &\subseteq& \Supported{S}
\\
S &=& \CanSupported{S} &\cap& \overline{\Defeated{S}} &\cap& \Supported{S}
\end{IEEEeqnarray*}
Moreover, since stable extensions are also \cfree\ sets~\cite[Theorem~2]{Fandinno2018foiks}, it follows that
\mbox{$\Supported{S} \subseteq \CanSupported{S}$}
and, thus,
\begin{IEEEeqnarray*}{ l ,C, rCcCl ,C, l ,C, l}
S &=& \overline{\Defeated{S}} &\cap& \Supported{S} &=& T_\sI^+ \setminus T_\sI^- &=& \SI
\end{IEEEeqnarray*}
Hence, $\SI$ is a stable extension and~\ref{item:2:thm:stable<->emodel} follows directly from Lemma~\ref{lem:stable->model}.
\\[10pt]
Let us show now that this determines a one-to-one correspondence.
Let $\sI$ and $\sJ$ be two equilibrium models such that $\SI = \SJ$.
Then, from Proposition~\ref{prop:ef.model.supportable},
$\CanSupported{\SI} = T_\sI^+$ and
$\CanSupported{\SJ} = T_\sJ^+$
and, thus, we have $T_\sI^+ = T_\sJ^+$.
Similarly, from Lemma~\ref{lem:min.model->defeated} 
it follows that 
$\Defeated{\SI} = T_\sI^+$ and
$\Defeated{\SJ} = T_\sJ^+$
and, thus, we have $T_\sI^- = T_\sJ^-$.
Hence, $\sI = \sJ$.
The other way around. Let $S_1$ and $S_2$ be two stable extensions such that
$\Supported{S_1} = \Supported{S_2}$ and
$\Defeated{S_1} = \Defeated{S_2}$.
Note that, since $S_1$ and $S_2$ are stable extensions,
we have that $S_i = \Supported{S_i} \setminus \Defeated{S_i}$ with $i \in \set{1,2}$ and, thus, $S_1 = S_2$.
\end{proofof}

\begin{proofof}{Proposition~\ref{prop:to.N4}}
    For any rule $r \in P$ of the form $H \supl B^{+} \wedge B^{-}$ with
    \begin{align*}
    H \ &= \ h_1 \vee \dots \vee h_k 
    \\
    B^{+} \ &= \ b_1 \wedge \dots \wedge b_m
    \\
    B^{-} \ &= \ \Not c_1 \wedge \dots \wedge \Not c_n
    \end{align*}
    we get that
    $\sI\sep w \modelsp r$
    \\iff
    $\sI\sep w \modelsp (B^{+} \wedge B^{-}) \wedge \neg\!\sneg\, (B^{+} \wedge B^{-}) \to H$
    \\iff either
    \\\phantom{or}
    $\sI\sep t \not\modelsp (B^{+} \wedge B^{-}) \wedge \neg\!\sneg\, (B^{+} \wedge B^{-})$
    \\or
    $\sI\sep w \modelsp H$
    \\or
    $\sI\sep w \not\modelsp (B^{+} \wedge B^{-}) \wedge \neg\!\sneg\, (B^{+} \wedge B^{-})$ and $\sI\sep t \modelsp H$.
\\[5pt]Furthermore, we also can see that
    $\sI\sep w \modelsp (B^{+} \wedge B^{-})$ holds
    \\iff
    $\sI\sep w \modelsp b_i$ for all $i \in \set{1,\dotsc,m}$
    and
    $\sI\sep w \modelsp \Not c_i$ for all $i \in \set{1,\dotsc,n}$
\\[5pt]Similarly, we can see that
    $\sI\sep t \modelsn (B^{+} \wedge B^{-})$ holds
    \\iff
    $\sI\sep t \modelsn b_i$ for some $i \in \set{1,\dotsc,m}$
    or
    $\sI\sep t \modelsn \Not c_i$ for some $i \in \set{1,\dotsc,n}$
\\[5pt]
    Hence, we get that
    $\sI\sep w \modelsp (B^{+} \wedge B^{-}) \wedge \neg\!\sneg\, (B^{+} \wedge B^{-})$
    \\iff
    $\sI\sep w \modelsp (B^{+} \wedge B^{-})$ and $\sI\sep t \not\modelsp \sneg\,(B^{+} \wedge B^{-})$
    \hfill (Proposition~\ref{prop:preservation})
    \\iff
    $\sI\sep w \modelsp (B^{+} \wedge B^{-})$ and $\sI\sep t \not\modelsn (B^{+} \wedge B^{-})$
    \\iff
    $\sI\sep w \modelsp b_i$ for all $i \in \set{1,\dotsc,m}$
    \\\phantom{iff }and
    $\sI\sep w \modelsp \Not c_i$ for all $i \in \set{1,\dotsc,n}$
    \\\phantom{iff }and
    $\sI\sep t \not\modelsn b_i$ for all $i \in \set{1,\dotsc,m}$
    \\\phantom{iff }and
    $\sI\sep t \not\modelsn \Not c_i$ for all $i \in \set{1,\dotsc,n}$
    \\iff
    $\sI\sep w \modelsp b_i$ for all $i \in \set{1,\dotsc,m}$
    \\\phantom{iff }and
    $\sI\sep t \modelsp \neg\sneg b_i$ for all $i \in \set{1,\dotsc,m}$
    \\\phantom{iff }and
    $\sI\sep w \modelsp \Not c_i$ for all $i \in \set{1,\dotsc,n}$
    \\\phantom{iff }and
    $\sI\sep t \not\modelsn \Not c_i$ for all $i \in \set{1,\dotsc,n}$
    \\iff
    $\sI\sep w \modelsp b_i \wedge \neg\!\sneg b_i$ for all $i \in \set{1,\dotsc,m}$
    \hfill (Proposition~\ref{prop:preservation})
    \\\phantom{iff }and
    $\sI\sep w \models \Not c_i$ for all $i \in \set{1,\dotsc,n}$
    \\iff
    $\sI\sep w \modelsp b_i \wedge \neg\!\sneg b_i$ for all $i \in \set{1,\dotsc,m}$
    \\\phantom{iff }and
    $\sI\sep w \modelsp \neg c_i \vee (c_i \wedge \sneg c_i)$ for all $i \in \set{1,\dotsc,n}$
    \hfill (Lemma~\ref{lem:default.negation})
    \\iff
    $\sI\sep w \modelsp C^{+} \wedge C^{-}$
    with
    \begin{align*}
    C^{+} \ &= \ b_1 \wedge \neg\!\sneg b_1 \wedge \dots \wedge b_m \wedge \neg\!\sneg b_1m
    \\
    C^{-} \ &= \ \neg c_1 \vee (c_1 \wedge \sneg c_1) \wedge \dots \wedge \neg  c_n \vee (c_n \wedge \sneg c_n)
    \end{align*}
    Therefore
    $\sI\sep w \modelsp r$
    \\iff either
    \\\phantom{or}
    $\sI\sep t \not\modelsp C^{+} \wedge C^{-}$
    \\or
    $\sI\sep w \modelsp H$
    \\or
    $\sI\sep w \not\modelsp C^{+} \wedge C^{-}$ and $\sI\sep t \modelsp H$.
    \\iff
    $\sI\sep w \modelsp C^{+} \wedge C^{-} \to H$
\end{proofof}

\begin{corollary}\label{cor:translation}
    Given a $\LanLP$-program $P$ and a total interpretation $\sI$, we have that
    $\sI$ is an equilibrium model of $P$ iff $\sI$ is an equilibrium model of $\LNp{P}$.
\end{corollary}

The following result is an adaptation of Theorem~1 by~\citeNP{CabalarFC0V17} and will be useful in proving the following result.

\begin{lemma}\label{lem:atom.def}
Let~$\Gamma$ be any theory without occurrences of strong negation, $a$ be an atom not occurring in~$\Gamma$ and~$\bT$ be a set of atoms.
Then, the following two statements are equivalent.
\begin{itemize}
    \item $\tuple{\bT,\bT}$ is an equilibrium model of~$\Gamma[\varphi/a] \cup \{ \varphi \to a \}$, and
    \item $\tuple{\bT',\bT'}$ is an equilibrium model of~$\Gamma$.
\end{itemize}
where~$\bT' = \bT \setminus\{a\}$.
\end{lemma}

\begin{proofof}{Proposition~\ref{prop:tr.lp}}
From Corollary~\ref{cor:translation},
it is enough to show that
\begin{gather}
    \text{$\tau\sI$ is an equilibrium model of~$\tau P$ iff $\sI$ is an equilibrium model of~$\delta P$}
    \label{eq:1:prop:tr.lp}
\end{gather}
Let~$P_1$ be the result of replacing every explicit literal of the form $\sneg a$ by a fresh atom~$\tilde{a}$
and, for any interpretation~$\sJ$, let~$\sJ_1$ be the total interpretation such that
\begin{enumerate}[ leftmargin=20pt ]
\item $\sJ_1 \not\modelsn a$ 
\item $\sJ_1 \modelsp a$ iff $\sJ \modelsp a$
\item $\sJ_1 \modelsp \tilde{a}$ iff $\sJ \modelsn a$
\end{enumerate}
Then, $\sJ \modelsp \delta P$ iff~$\sJ_1 \models P_1$ for any interpretation~$\sJ$ and, thus,
\begin{gather}
    \text{$\sI$ is an equilibrium model of~$\delta P$ iff $\sI_1$ is an equilibrium model of~$P_1$}
\end{gather}
and, we can rewrite~\eqref{eq:1:prop:tr.lp} as
\begin{gather}
    \text{$\tau\sI$ is an equilibrium model of~$\tau P$ iff $\sI_1$ is an equilibrium model of~$P_1$}
    \label{eq:2:prop:tr.lp}
\end{gather}
Let now~$P_2$ be the result of replacing each occurrence of~$\neg a \vee (a \wedge \tilde{a})$ in the body of any rule by~$a'$ and add the rule~$a' \leftarrow \neg a \vee (a \wedge \tilde{a})$.
Then, from Lemma~\ref{lem:atom.def}, we get that
\begin{gather*}
    \text{$\sI_1$ is an equilibrium model of~$P_1$ iff $\tau\sI$ is an equilibrium model of~$P_2$}
\end{gather*}
Finally, the result follows by noting that~$a' \leftarrow \neg a \vee (a \wedge \tilde{a})$
is equivalent to the conjunction of formulas~$a' \leftarrow \neg a$ and~$a' \leftarrow a \wedge \tilde{a}$
and that this transformations applied to~$P_2$ yields~$\tau P$.\end{proofof} 
\end{document}